\documentclass[english, 11pt]{article}
\usepackage{mathpazo}
\usepackage[T1]{fontenc}
\usepackage[latin9]{inputenc}
\usepackage{geometry}
\usepackage{babel}
\usepackage{float}
\usepackage{url}
\usepackage{amsmath}
\usepackage{amsthm}
\usepackage[authoryear]{natbib}
\usepackage{amssymb}
\usepackage{undertilde}
\usepackage{algpseudocode}
\usepackage{xspace}
\usepackage{microtype}
\usepackage{enumitem}
\algrenewcommand\algorithmicindent{2.0em}%

\usepackage[unicode=true,pdfusetitle, bookmarks=true,bookmarksnumbered=false,bookmarksopen=false,breaklinks=true,pdfborder={0 0 1},backref=false,colorlinks=true, citecolor=blue] {hyperref}

\makeatletter

\floatstyle{ruled}
\newfloat{algorithm}{tbp}{loa}
\providecommand{\algorithmname}{Algorithm}
\floatname{algorithm}{\protect\algorithmname}

\theoremstyle{plain}
\newtheorem{assumption}{\protect\assumptionname}
\theoremstyle{definition}
\newtheorem{defn}{\protect\definitionname}
\theoremstyle{plain}
\newtheorem{lem}{\protect\lemmaname}
\theoremstyle{plain}
\newtheorem{thm}{\protect\theoremname}
\theoremstyle{plain}
\newtheorem{cor}{\protect\corollaryname}

\usepackage{appendix,breakurl}
\PassOptionsToPackage{hyphens}{url}

\makeatother

\providecommand{\assumptionname}{Assumption}
\providecommand{\corollaryname}{Corollary}
\providecommand{\definitionname}{Definition}
\providecommand{\lemmaname}{Lemma}
\providecommand{\theoremname}{Theorem}

\begin{document}

\global\long\def\mS{\mathcal{S}}%
\global\long\def\mA{\mathcal{A}}%
\global\long\def\mP{\mathcal{P}}%
\global\long\def\mQ{\mathcal{Q}}%
\global\long\def\mV{\mathcal{V}}%
\global\long\def\E{\mathbb{E}}%
\global\long\def\R{\mathbb{R}}%
\global\long\def\P{\mathrm{P}}%

\global\long\def\simplex{\Delta}%
\global\long\def\mL{\mathcal{L}}%
\global\long\def\bL{\boldsymbol{\mathcal{L}}}%
\global\long\def\H{\mathcal{H}}%

\global\long\def\mH{\mathcal{H}}%
\global\long\def\Qhat{\widehat{Q}}%
\global\long\def\pihat{\widehat{\pi}}%
\global\long\def\tildeO{\widetilde{O}}%

\global\long\def\bigO{\mathcal{O}}%
 
\global\long\def\indic{\mathds{1}}%
\global\long\def\joint{\mathcal{M}}%
\global\long\def\muspace{\mathcal{M}}%
\global\long\def\mdp{\mathrm{MDP}}%
\global\long\def\entro{\mathbb{H}}%
\global\long\def\KL{D_{\mathrm{KL}}}%
\global\long\def\ddup{\mathrm{d}}%
\global\long\def\KLmax{\mathrm{KL}_{\max}}%

\global\long\def\gapmu{\sigma_{\mu}}%
\global\long\def\gappi{\sigma_{\pi}}%
\global\long\def\argmax{\operatorname*{argmax}}%

\title{Provable Fictitious Play for General Mean-Field Games}

\author{Qiaomin Xie$^\dagger$, Zhuoran Yang$^\mathsection$, Zhaoran Wang$^\ddagger$, Andreea Minca$^\dagger$  \footnote{Emails: \texttt{qiaomin.xie@cornell.edu}, \texttt{zy6@princeton.edu}, \texttt{zhaoranwang@gmail.com}, \texttt{acm299@cornell.edu}}\\ ~\\
	\normalsize $^\dagger$School of Operations Research and Information Engineering, Cornell University\\
	\normalsize $^\mathsection$Department of Operations Research and Financial Engineering, Princeton University\\
	\normalsize $^\ddagger$Department of Industrial Engineering and Management Sciences, Northwestern University
}

\date{}

\maketitle
\begin{abstract}
We propose a reinforcement learning algorithm for stationary mean-field
games, where the goal is to learn a pair of mean-field state and stationary
policy that constitutes the Nash equilibrium. When viewing the mean-field
state and the policy as two players, we propose a fictitious play
algorithm which alternatively updates the mean-field state and the
policy via gradient-descent and proximal policy optimization, respectively.
Our algorithm is in stark contrast with previous literature which
solves each single-agent reinforcement learning problem induced by
the iterates mean-field states to the optimum. Furthermore, we prove
that our fictitious play algorithm converges to the Nash equilibrium
at a sublinear rate. To the best of our knowledge, this seems the
first provably convergent single-loop reinforcement learning algorithm for mean-field
games based on iterative updates of both mean-field state and policy.

\end{abstract}

\section{Introduction}

Multi-agent reinforcement learning (MARL) \citep{shoham2007if, busoniu2008comprehensive,  hernandez2017survey, hernandez2018multiagent, zhang2019multi}   aims to tackle  sequential decision-making problems in  multi-agent systems \citep{wooldridge2009introduction} by integrating  the classical reinforcement learning framework   \citep{sutton2018reinforcement}  with game-theoretical thinking \citep{bacsar1998dynamic}. 
Powered by deep-learning \citep{goodfellow2016deep}, MARL recently has achieved striking empirical successes in games \citep{silver2016alphago,silver2017alphagozero,vinyals2019grandmaster,  berner2019dota, schrittwieser2019mastering}, robotics \citep{yang2004multiagent, busoniu2006decentralized, leottau2018decentralized},  transportation \citep{kuyer2008multiagent,mannion2016experimental}, and 
 social science \citep{leibo2017multi, jaques2019social, cao2018emergent, mckee2020social}.

Despite the  empirical successes,  
MARL is known to suffer from the scalability issue. 
Specifically, 
in a multi-agent system, each agent interacts with the other agents as well as the environment, with the goal of maximizing its own expected total return.
As a result, for each agent, the reward function and the transition kernel  of its local state also involve the local states and actions of all the other agents. 
As a result,  
as the number of agents  increases, the capacity of the joint state-action space grows exponentially, 
which brings tremendous difficulty to reinforcement learning algorithms due to the need to handle  high-dimensional  input spaces. 
Such a curse of dimensionality due to having a large number of agents in the system is named as the ``curse of many agents'' \citep{sonu2017decision}.

To circumvent such a notorious  curse, 
a popular approach is through mean-field approximation, which imposes symmetry among the agents and specifies that, for each agent, the joint effect of all the other agents is summarized by a population quantity, which is oftentimes given by the empirical distribution of the local states and actions of all the other agents or a functional of such an empirical distribution.   
Specifically, to obtain symmetry, the reward and local state transition functions  are the same for each agent, which are functions of the local state-action and the population quantity. 
Thanks to mean-field approximation, such a multi-agent system, known as the mean-field game (MFG) \citep{huang2003individual, lasry2006jeux, lasry2006jeux1, lasry2007mean, huang2007large, gueant2011mean, carmona2018probabilistic}, is   readily scalable to  an arbitrary number of agents.

In this work, we  aim to find the Nash equilibrium \citep{nash1950equilibrium} of MFG with infinite number of agents via reinforcement learning. 
By mean-field approximation, such a game consists of a population of symmetric agents among which each individual agent has infinitesimal  effect over the whole population. 
By symmetry, it suffices to find a symmetric Nash equilibrium where each agent adopts the same policy.  
Under such consideration, we can focus on a single agent, also known as the representative agent, and 
view  MFG as a game between the representative agent's local policy 
  $\pi    $ and the mean-field state  $\mathcal{L} $ which aggregates the collective effect of the  population.
Specifically, the representative agent  $\pi$ aims to find the optimal policy when the mean-field state is fixed to $\mathcal{L}$, which reduces  to solving a Markov decision process (MDP) induced by $\mathcal{L}$. 
Simultaneously, we aim to let $\mathcal{L}$ be the mean-field state when all the agents adopt policy $\pi$. 
The Nash equilibrium of such a two-player game,  $(\pi^*, \mathcal{L}^*)$,  yields a symmetric  Nash equilibrium $\pi^*$ of the original MFG.

Under proper conditions,  
the Nash equilibrium $(\pi^*, \mathcal{L}^*)$ can be obtained via   fixed-point updates, which generate  a sequence $\{ \pi_t, \mathcal{L}_t  \}$ as follows. 
  For any $t\geq 0$, in the $t$-th iteration, we solve the MDP induced by $\mathcal{L}_t$ and let $\pi_t$ be the optimal policy. Then we update the mean-field state by letting $\mathcal{L}_{t+1}$ be the mean-field state obtained by letting every agent follow $\pi_t$. 
  Under appropriate assumptions, the mapping from  $ \mathcal{L}_t $ to $ \mathcal{L}_{t+1}  $ is a contraction and thus such an iterative algorithm converges to the unique fixed-point of such a contractive mapping, which corresponds to $\mathcal{L}^*$ \citep{guo2019learning}. 
  Based on the contractive property, various reinforcement learning methods are proposed to approximately implement the fixed-point updates and   find the Nash equilibrium $(\pi^*, \mathcal{L}^*)$  \citep{guo2019learning,guo2020general, anahtarci2019value, anahtarci2019fitted, anahtarci2020q}. 
However,  such an approach requires solving a standard reinforcement learning problem 
approximately within each iteration, which itself is solved by an iterative algorithm such as Q-learning \citep{watkins1992q,mnih2015human, bellemare2017distributional} or   actor-critic methods \citep{konda2000actor,haarnoja2018soft, schulman2015trust, schulman2017proximal}. 
As a result, this approach leads to a \emph{double-loop} iterative algorithm for solving MFG. 
When the state space $\mathcal{S}$ is enormous, 
function approximation tools such as deep neural networks are equipped to represent the value and policy functions in the reinforcement learning algorithm, making solving each inner  subproblem 
computationally demanding. 

To obtain a computationally efficient algorithm for MFG, we consider the following question: 
\begin{center}
    Can we design a single-loop  reinforcement learning algorithm for solving MFG which updates the policy and mean-field state simultaneously in each iteration? 
\end{center}
For such a question, we  provide an affirmative answer by proposing a fictitious play \citep{brown1951iterative}  policy optimization algorithm, where we view the policy $\pi$ and mean-field state $\mathcal{L}$ as the two players and update them simultaneously in each iteration. 
Fictitious play is a general algorithm framework for solving games where each player first infers the opponent and then improves its own policy based on the inferred opponent information. 
When it comes to MFG, in each iteration, the policy player $\pi$ first infers the mean-field state implicitly by solving a policy evaluation problem associated with $\pi$ on the MDP induced by $\mathcal{L}$. Then the policy $\pi$ is updated   via a proximal policy optimization (PPO) \citep{schulman2017proximal} step with entropy regularization, which is adopted to ensure the uniqueness of the Nash equilibrium. 
Meanwhile,   the mean-field state $\mathcal{L}$ obtains its update direction by solving how the mean-field state evolves when all the agents execute  policy $\pi$ with their state distribution being $\mathcal{L}$. Then $\mathcal{L}$ is updated towards this direction with some stepsize. Such an algorithm is single-loop as the mean-field state $\mathcal{L}$ is updated immediately when $\pi$  is updated. 

Furthermore, since $\mathcal{L}$ is a distribution over the state space $\mathcal{S} $, when $\mathcal{S}$ is continuous,  $\mathcal{L} $ lies in an  infinite-dimensional space, which makes it computationally challenging to be updated.  
To overcome this challenge, we employ a succinct representation of $\mathcal{L}$ via kernel mean embedding, which maps $\mathcal{L}$ to an element in a reproducing kernel Hilbert space (RKHS) \citep{smola2007hilbert, gretton2008kernel, sriperumbudur2010hilbert}. Such a mechanism  enables us to update the mean-field state within RKHS, which can be computed efficiently. 

When the stepsizes for policy and mean-field state updates are properly chosen, we prove that our single-loop fictitious play algorithm   converges to the entropy-regularized Nash equilibrium at a sublinear $   \widetilde{\mathcal{O}} ( T^{-1/5})$-rate, where $T$ is the total number of iterations and $  \widetilde{\mathcal{O}} (\cdot)$ hides logarithmic terms. 
To our best knowledge, we establish the first single-loop reinforcement learning algorithm for mean-field game with finite-time convergence guarantee to Nash equilibrium. 
   
\paragraph{Our Contributions.} Our contributions are two-fold. 
First, we propose a single-loop fictitious play algorithm that updates both the policy and the mean-field state simultaneously in each iteration, where the policy is updated via  entropy-regularized proximal policy optimization. 
Moreover, 
we utilize kernel mean embedding to represent  the mean-field states and the policy update subroutine can readily incorporate any function approximation tools to represent both the value and policy functions, which makes our fictitious play method  a general algorithmic framework that is able to handle MFG with continuous state space. 
Second, we prove that the   policy and mean-field state sequence generated by the proposed  algorithm converges to the Nash equilibrium of the MFG at a sublinear $\widetilde{ \mathcal{O}}(T^{-1/5})$ rate.

\paragraph{Related Works.}

Our work belongs to the literature on discrete-time MFG. 
A variety of works have focused on the 
existence of a Nash equilibrium and the behavior of Nash equilibrium as the number of agents goes to infinity under various settings of MFG. See, e.g.,  \cite{gomes2010discrete, tembine2011mean, moon2014discrete,biswas2015mean,saldi2018markov, saldi2018discrete,saldi2019approximate, wikecek2020discrete} and the references therein. 
In addition, our work is more related to the line of research that aims to solve MFG via reinforcement learning methods. 
Most of the existing works propose to find the Nash equilibrium via fixed-point iterations in space of the mean-field states, which requires solving an MDP induced by a mean-field state within each iteration \citep{guo2019learning,guo2020general, anahtarci2019fitted, anahtarci2019value, fu2019actor, uz2020approximate, anahtarci2020q}. 
Among these works, \cite{guo2019learning,guo2020general, anahtarci2019fitted, anahtarci2019value,anahtarci2020q} propose to solve each MDP via Q-learning \citep{watkins1992q} or approximated value iteration \citep{munos2008finite}, whereas \cite{fu2019actor, uz2020approximate} solve each MDP using actor-critic \citep{konda2000actor} under the linear-quadratic setting. Furthermore,  more closely related works are  \cite{elie2020convergence, perrin2020fictitious}, which study the convergence of a version of fictitious play for MFG. 
Similar to our algorithm, their fictitious play also regards the policy and the mean-field state as the two players.
 However, for policy update, they compute the best response policy to the current mean-field state  by solving the MDP induced by the mean-field state to approximate optimality, and the obtained policy is added to the set of previous policy iterates to form a mixture policy. 
As a result, their algorithm is double-loop in essence due to solving an MDP in each iteration. 
In contrast, our fictitious play is single-loop --- the policy is updated via a single PPO step in each iteration, and the mean-field state is updated before the policy solves any MDP associated with a mean-field state. We remark that the recent work by \cite{subramanian2019reinforcement} also considers a single-loop algorithm. However, only asymptotic convergence guarantee is established via two time-scale stochastic approximation.

\paragraph{Notations.}

We use $\left\Vert \cdot\right\Vert _{1}$ to denote the vector $\ell_{1}$-norm,
and $\simplex(\mathcal{\mathcal{D}})$ the probability simplex over
$\mathcal{D}.$ The Kullback-Leibler (KL) divergence between $p_{1},p_{2}\in\simplex(\mA)$
is defined as $\KL(p_{1}\Vert p_{2}):=\sum_{a\in\mA}p_{1}(a)\log\frac{p_{1}(a)}{p_{2}(a)}.$
Let $\boldsymbol{1}_{n}\in\R^{n}$ denote the all-one vector.  For
two quantities $x$ and $y$ that may depend on problem parameters
($|\mA|,\gamma,$ etc.), if $x\ge Cy$ holds for a universals constant
$C>0$, we write $x\gtrsim y$, $x=\Omega(y)$ and $y=\bigO(x)$.
We use $\tildeO(\cdot)$ to denote $\bigO(\cdot)$ ignoring logarithmic
factors.

\section{Background and Preliminaries\label{sec:prelim}}

In this section, we first review the standard setting of mean-field
games (MFG) from \citet{guo2019learning}, and then introduce a more
general MFG with mean embedding and entropy regularization. 

\subsection{Mean-Field Games\label{sec:mean_field_games}}

Consider a discrete-time Markov game involving an infinite number
of identical and interchangeable agents. Let $\mS\subseteq\R^{d}$
and $\mA\subseteq\R^{p}$ be the state space and action space, respectively,
that are common to the agents. We assume that $\mS$ is compact and
$\mA$ is finite. The reward and the state dynamic for each agent
depend on the collective behavior of all agents through the mean-field
state, i.e., the \emph{distribution} of the states of all agents.
As the agents are homogeneous and interchangeable, one can focus on
a single agent representative of the population. Let $r:\mS\times\mA\times\simplex(\mS)\to[0,R_{\max}]$
be the (bounded) reward function and $\P$: $\mS\times\mA\times\simplex(\mS)\to\simplex(\mS)$
be the state transition kernel. At each time $t$, the representative
agent is in state $s_{t}\in\mS$, and the probability distribution
of $s_{t}$, denoted by $\mL_{t}\in\simplex(\mS)$, corresponds to
the mean-field state. Upon taking an action $a_{t}\in\mA$, the agent
receives a reward $r(s_{t},a_{t},\mL_{t})$ and transitions to a new
state $s_{t+1}\sim\P(\cdot|s_{t},a_{t},\mL_{t})$. A Markovian policy
for the agent is a function $\pi:\mS\to\simplex(\mA)$ that maps her
own state to a distribution over actions,\footnote{In general, the policy may be a function of the mean-field state $\mL_{t}$
as well. We have suppressed this dependency since our ultimate goal
is to find a \emph{stationary} equilibrium, under which the mean-field
state remains fixed over time. See \citet{guo2019learning,saldi2018markov}
for a similar treatment.} i.e., $\pi(a|s)$ is the probability of taking action $a$ in state
$s$. Let $\Pi$ be the set of all Markovian policies.

When an agent is operating under a policy $\pi\in\Pi$ and the mean-field
population flow is $\bL:=(\mL_{t})_{t\geq0}$, we define the expected
cumulative discounted reward (or value function) of this agent as
\begin{align*}
	V^{\pi}(s,\boldsymbol{\mL}):=\E\bigg[{ \sum_{t=0}^{\infty} } \gamma^{t}r(s_{t},a_{t},\mL_{t})\mid s_{0}=s\bigg],
\end{align*}
where $a_{t}\sim\pi(\cdot|s_{t}),$ $s_{t+1}\sim\P(\cdot|s_{t},a_{t},\mL_{t})$,
and $\gamma\in(0,1)$ is the discount factor. The goal of this agent
is to find a policy $\pi$ that maximizes $V^{\pi}(s,\boldsymbol{\mL})$
while interacting with the mean-field $\bL$. 

We are interested in finding a \emph{stationary (time-independent)
Nash Equilibrium} (NE) of the game, which is a policy-population pair
$(\pi^{*},\mL^{*})\in\Pi\times\simplex(\mS)$ satisfying the following
two properties:
\begin{itemize}[leftmargin=7mm]
\item (Agent rationality) $V^{\pi^{*}}(s,\mL^{*})\geq V^{\pi}(s,\mL^{*}),\forall\pi\in\Pi,s\in\mS.$
\item (Population consistency) $\mL_{t}=\mL^{*},\forall t$ under policy
$\pi^{*}$ with initial mean-field state $\mL_{0}=\mL^{*}$.
\end{itemize}
That is, $\pi^{*}$ is the optimal policy under the mean-field $\mL^{*}$,
and $\mL^{*}$ remains fixed under $\pi^{*}$. We formalize the notion
of NE in Section~\ref{sec:regularization} after introducing a more
general setting of MFG.

\subsection{Mean Embedding of Mean-Field States\label{sec:mean_embedding}}

\textbf{}Note that the mean-field state $\mL^{*}$ is a distribution
over the states. When the state space is continuous, the NE $(\pi^{*},\mL^{*})$
is an infinite dimensional object, posing challenges for learning
the NE. To overcome this challenge, we make use of a succinct representation
of the mean-field via mean embedding, which embeds the mean-field
states into a reproducing kernel Hilbert space (RKHS) \citep{smola2007hilbert,gretton2008kernel,sriperumbudur2010hilbert}.
 Specifically, given a positive definite kernel $k:\mS\times\mS\rightarrow\R$,
let $\H$ be the associated RKHS endowed with the inner product $\left\langle \cdot,\cdot\right\rangle _{\mH}$
and norm $\left\Vert \cdot\right\Vert _{\mH}$. For each $\mL\in\simplex(\mS)$,
its mean embedding $\mu_{\mL}\in\mH$ is defined as 
\[
\mu_{\mL}(s):=\E_{x\sim\mL}\left[k(x,s)\right],\quad\forall s\in\mS.
\]
Let $\muspace:=\left\{ \mu_{\mL}:\mL\in\simplex(\mS)\right\} \subseteq\mH$
be the set of all possible mean embeddings. Note that when $k$ is
the identity kernel, we have $\mu_{\mL}=\mL$ and $\muspace=\simplex(\mS)$
. On the other hand, when $k$ is more structured (e.g., with a fast
decaying eigen spectrum), $\muspace$ has significantly lower complexity
than the set $\simplex(\mS)$ of raw mean-field states.

We assume that the MFG respects the mean embedding structure, in the
sense that the reward $r:\mS\times\mA\times\muspace\to[0,R_{\max}]$
and transition kernel $\P:\mS\times\mA\times\muspace\to\simplex(\mS)$
(with a slight abuse of notation) depend on the mean-field state $\mL$
through its mean embedding representation $\mu_{\mL}$. In particular,
at each time $t$ with state $s_{t}$ and mean-field state $\mL_{t}$,
the representative agent takes action $a_{t}\sim\pi(\cdot|s_{t})$, receives
reward $r(s_{t},a_{t},\mu_{\mL_{t}})$ and then transitions to a new
state $s_{t+1}\sim\P(\cdot|s_{t},a_{t},\mu_{\mL_{t}})$. The NE of
the game is defined analogously. As mentioned, when $k$ is the identity
kernel, the above setting reduces to the standard setting in Section~\ref{sec:mean_field_games}
with raw-mean field states.

We impose a standard regularity condition on the kernel $k$.
\begin{assumption}
\label{assu:RKSH_kernel}The kernel $k$ is bounded and universal,
in the sense that $k(s,s)\leq1,\forall s\in\mS$ and the corresponding
RKHS $\mH$ is dense w.r.t.~the $L_{\infty}$ norm in the space of
continuous functions on $\mS$.
\end{assumption}
Assumption \ref{assu:RKSH_kernel} is standard in the kernel learning
literature \citep{caponnetto2007optimal,muandet2012learning,szabo2015two,lin2017distributed}.
 When the kernel is bounded, the embedding of each $\mL\in\simplex(\mS)$
satisfies $\left\Vert \mu_{\mL}\right\Vert _{\mH}\leq\int_{x\sim\mL}\left\Vert k(x,\cdot)\right\Vert _{\mH}\textup{d}x\leq1.$
When one uses a universal kernel (e.g., Gaussian or Laplace kernel),
the mean embedding mapping is injective and hence each embedding $\mu\in\muspace$
uniquely characterizes a distribution $\mL$ in $\simplex(\mS)$ \citep{gretton2008kernel,gretton2012kernel}.

\subsection{Entropy Regularization\label{sec:regularization}}

 To ensure the uniqueness of the NE and achieve fast algorithmic
convergence, we use an entropy regularization approach \citep{cen2020fast,shani2019adaptive,nachum2017bridging},
which augments the standard expected reward objective with an entropy
term of the policy. In particular, we define the entropy-regularized
value function as 
\begin{align*}
V_{\mu}^{\lambda,\pi}(s):=\E_{a_{t}\sim\pi(\cdot|s_{t}),s_{t+1}\sim\P(\cdot|s_{t},a_{t},\mu)}\bigg[ { \sum_{t=0}^{\infty} } \gamma^{t} [r(s_{t},a_{t},\mu)-\lambda\log\pi(a_{t}|s_{t}) ]\mid s_{0}=s\bigg],
\end{align*}
where the parameter $\lambda>0$ controls the regularization level
and $\mu$ is the mean-embedding of some given mean-field state (fixed
over time). Equivalently, one may view $V_{\mu}^{\lambda,\pi}$ as
the usual value function of $\pi$ with an entropy-regularized reward
\begin{equation}
r_{\mu}^{\lambda,\pi}(s,a):=r(s,a,\mu)-\lambda\log\pi(a|s),\qquad\forall s\in\mS,a\in\mA.\label{eq:ER_reward}
\end{equation}
Also define the $Q$-function of a policy $\pi$ as
\begin{align}
Q_{\mu}^{\lambda,\pi}(s,a) & =r(s,a,\mu)+\gamma\E\left[V_{\mu}^{\lambda,\pi}(s_{1})\mid s_{0}=s,a_{0}=a\right],\label{eq:ER_Q}
\end{align}
which is related to the value function as
\begin{align}
V_{\mu}^{\lambda,\pi}(s) & =\E_{a\sim\pi(\cdot|s)}\left[Q_{\mu}^{\lambda,\pi}(s,a)-\lambda\log\pi(a|s)\right]=\left\langle Q_{\mu}^{\lambda,\pi}(s,\cdot),\pi(\cdot|s)\right\rangle +\entro\left(\pi(\cdot|s)\right),\label{eq:ER_V_ER_Q}
\end{align}
where $\entro\left(\pi(\cdot|s)\right):=-\sum_{a}\pi(a|s)\log\pi(a|s)$
is the Shannon entropy of the distribution $\pi(\cdot|s)$. Since
the reward function $r$ is assumed to be $R_{\max}$-bounded, it
is easy to show that the Q-function is also bounded as $\left\Vert Q_{\mu}^{\lambda,\pi}\right\Vert _{\infty}\le Q_{\max}:= (R_{\max}+\gamma\lambda\log\left|\mA\right| ) / (1-\gamma);$
see Lemma~\ref{lem:optimal_ER_MDP}.

\paragraph{Single-Agent MDP.}

When the mean-field state and its mean-embedding remain fixed over
time, i.e., $\mL_{t}=\mL$ and $\mu_{\mL_{t}}=\mu,\forall t$, a representative
agent aims to solve the optimization problem 
\begin{align}
\max_{\pi:\mS\rightarrow\simplex(\mA)} & V_{\mu}^{\lambda,\pi}(s)\label{eq:ER_MDP_L}
\end{align}
for each $s\in\mS$. This problem corresponds to finding the (entropy-regularized)
optimal policy for a single-agent discounted MDP, denoted by $\mdp_{\mu}:=\left(\mS,\mA,\P(\cdot|\cdot,\cdot,\mu),r(\cdot,\cdot,\mu),\gamma\right)$,
that is induced by $\mu\in\muspace$. Let $\pi_{\mu}^{\lambda,*}$
be the optimal solution to the problem~(\ref{eq:ER_MDP_L}), that
is, the optimal regularized policy of $\mdp_{\mu}$. The optimal policy
is unique whenever $\lambda>0$. One can thus define a mapping $\Gamma_{1}^{\lambda}:\muspace\to\Pi$
via
$
\Gamma_{1}^{\lambda}(\mu)=\pi_{\mu}^{\lambda,*}\,,
$
which maps each embedded mean-field state $\mu$ to the optimal regularized
policy $\pi_{\mu}^{\lambda,*}$ of $\mdp_{\mu}$. Let $Q_{\mu}^{\lambda,*}$
be the optimal regularized Q-function corresponding to the optimal
policy $\pi_{\mu}^{\lambda,*}$. 

Throughout the paper, we fix a state distribution $\nu_{0}\in\simplex(\mS)$,
which will serve as the initial state of our policy optimization algorithm.
For each $\mu\in\muspace$ and a policy $\pi:\mS\rightarrow\simplex(\mA)$,
define 
\begin{equation}
J_{\mu}^{\lambda}(\pi):=\E_{s\sim\nu_{0}}\left[V_{\mu}^{\lambda,\pi}(s)\right]\label{eq:expected_value_function}
\end{equation}
as the expectation of the value function $V_{\mu}^{\lambda,\pi}(s)$
of policy $\pi$ on the regularized $\mdp_{\mu}$. We define the discounted
state visitation distribution $\rho_{\mu}^{\pi}$ induced by a policy
$\pi$ on $\mdp_{\mu}$ as:
\begin{equation}
\vspace{-0.1in}
\rho_{\mu}^{\pi}(s):=(1-\gamma)\sum_{t=0}^{\infty}\gamma^{t}\mathbb{P}(s_{t}=s),\label{eq:def_state_visitation_dist}
\end{equation}
where $\mathbb{P}(s_{t}=s)$ is the state distribution when $s_{0}\sim\nu_{0}$
and the actions are chosen according to $\pi$.

\paragraph{Mean-field Dynamics.}

When all agents follow the same policy $\pi$, we can define another
mapping $\Gamma_{2}:\Pi\times\muspace\rightarrow\muspace$ that describes
the dynamic of the embedded mean-field state. In particular, given
the current embedding $\mu$ corresponding to some mean-field state
$\mL$, the next embedded mean-field state $\mu^{+}=\Gamma_{2}(\pi,\mu)$
is given by 
\begin{equation}
\mL^{+}(s')=\int_{\mS}\sum_{a\in\mA}\mL(s)\pi(a|s)\P(s'|s,a,\mu)\ddup s,\qquad\mu^{+}=\mu_{\mL^{+}}\,.\label{eq:mean_Gamma2}
\end{equation}
Note that the evolution of the mean-field depends on the agents' policy
in a deterministic manner.

\paragraph{Entropy-regularized Mean-field Nash Equilibrium (NE).}

With the above notations, we can formally define our notion of equilibrium.
\begin{defn}
A stationary (time-independent) entropy-regularized Nash equilibrium
for the MFG is a policy-population pair $(\pi^{*},\mu^{*})\in\Pi\times\muspace$
that satisfies 
\begin{align*}
\text{(agent rationality)}\qquad\pi^{*} & =\Gamma_{1}^{\lambda}(\mu^{*}),\\
\text{(population consistency)}\qquad\mu^{*} & =\Gamma_{2}(\pi^{*},\mu^{*}).\qquad\qquad\qquad\qquad\qquad
\end{align*}
\end{defn}
When $\lambda=0$, the above definition reduces to that of the (unregularized)
NE discussed in Section~\ref{sec:mean_field_games}, which requires
$\pi^{*}$ to the unregularized optimal policy of $\mdp_{\mu^{*}}$.
For general values of $\lambda$, the regularized NE $(\pi^{*},\mu^{*})$
approximates the unregularized NE \citep{geist2019theory}, in the
sense that $\pi^{*}$ is an approximate optimal policy of $\mdp_{\mu^{*}}$
satisfying 
\begin{align}
\max_{\pi\in\Pi}\{ J_{\mu^{*}}^{0}(\pi)\} -J_{\mu^{*}}^{\lambda}(\pi^{*})\le \frac{\lambda\log |\mA |}{1-\gamma}. \label{eq:ER_optimality_gap}
\end{align}

One may further define the composite mapping $\Lambda^{\lambda}:\muspace\to\muspace$
as $\Lambda^{\lambda}(\mu)=\Gamma_{2}\left(\Gamma_{1}^{\lambda}(\mu),\mu\right).$
When $\Lambda^{\lambda}$ is a contraction, the regularized NE exists and is unique \citep{guo2019learning}.
Moreover, the iterates $\left\{ (\pi_{t},\mu_{t})\right\} _{t\ge0}$
given by the two-step update
\[
\pi_{t}=\Gamma_{1}^{\lambda}(\mu_{t}),\qquad\mu_{t+1}=\Gamma_{2}(\pi_{t},\mu_{t})
\]
converge to the regularized NE at a linear rate. Note that the first
step above requires an oracle for computing the exact optimal policy
$\pi_{\mu_{t}}^{\lambda,*}$. In most cases, such an exact oracle
is not available; various single-agent reinforcement learning algorithms
have been considered for computing an approximate optimal policy,
including Q-learning \citep{guo2019learning} and policy gradient
methods \citep{guo2020general,subramanian2019reinforcement}. The
recent work by \citet{elie2020convergence} considers fictitious play
iterative learning scheme. We remark that their convergence guarantee
requires being able to compute the approximate optimal policy to an
arbitrary precision with high probability.

\section{Fictitious Play Algorithm for MFG}

In this section, we present a fictitious play algorithm, which simultaneously
estimates the policy $\pi^{*}$ and the embedded mean-field state
$\mu^{*}$ of the NE. As given in Algorithm \ref{alg:Fictitious-Play-embedded},
each iteration of the algorithm involves three steps: policy evaluation
(line 3), policy improvement (line 4), and updating the embedded mean-field
state (line 5). Below we explain each step in more details.

\begin{algorithm}[h]
\caption{\label{alg:Fictitious-Play-embedded}Mean-Embedded Fictitious Play}

\begin{algorithmic}[1]

\State Input: initial estimate $(\pi_{0},\mu_{0})$, step size sequence
$\{\alpha_{t},\beta_{t}\}_{t\geq0}$, mixing parameter $\eta$.

\For{Iteration $t=0,1,2,\ldots,T-1$ }

\State (Policy evaluation step) Compute an approximate version $\Qhat_{t}^{\lambda}:\mS\times\mA\to[0,Q_{\max}]$
of the \mbox{~~~~~~}Q-function $Q_{\mu_{t}}^{\lambda,\pi_{t}}$ of policy $\pi_{t}$
with respect to the entropy-regularized $\mdp_{\mu_{t}}$

\State (Policy improvement step) Update the policy by 
\begin{align}
\pihat_{t+1}(\cdot|s) & \propto\left(\pi_{t}(\cdot|s)\right)^{1-\alpha_{t}\lambda}\exp \big(\alpha_{t}\Qhat_{t}^{\lambda}(s,\cdot)\big)\label{eq:policy_update}\\
\pi_{t+1}(\cdot|s) & =(1-\eta)\pihat_{t+1}(\cdot|s)+\eta \cdot  \boldsymbol{1}_{|\mA|}(\cdot) / |\mA| \label{eq:policy_mix}
\end{align}

\State Update the embedded mean-field state by
\begin{equation}
\mu_{t+1}=(1-\beta_{t})\mu_{t}+\beta_{t}\cdot\Gamma_{2}(\pi_{t+1},\mu_{t}).\label{eq:mean_field_update}
\end{equation}

\EndFor

\State Output: $\left\{ (\ensuremath{\pi_{t},}\ensuremath{\mu_{t}})\right\} _{t=1,\ldots,T}$

\end{algorithmic}
\end{algorithm}

\paragraph{Policy Evaluation. }

In each iteration, we first evaluate the current policy $\pi_{t}$
with respect to the regularized single-agent $\mdp_{\mu_{t}}$ induced
by the current mean-field estimate $\mu_{t}$. In particular, we compute
an approximation $\Qhat_{t}^{\lambda}$ of the true Q-function $Q_{t}^{\lambda}:=Q_{\mu_{t}}^{\lambda,\pi_{t}}$,
which can be done using, e.g., TD(0) or LSTD methods. Our theorem
characterizes how convergence depends on the policy evaluation error
in this step.\textbf{}

\paragraph{Policy Improvement.}

To update our policy estimate $\pi_{t}$, we first compute an intermediate
policy $\pihat_{t+1}$ by a \emph{single} policy improvement step:
for each $s\in\mS$, 
\begin{equation}
\pihat_{t+1}(\cdot|s)=\argmax_{\pi(\cdot|s)\in\simplex(\mA)}\Big\{ \alpha_{t}\big\langle \Qhat_{t}^{\lambda}(s,\cdot)-\lambda\log\pi_{t}(\cdot|s),\pi(\cdot|s)-\pi_{t}(\cdot|s)\big\rangle -\KL\left(\pi(\cdot|s)\Vert\pi_{t}(\cdot|s)\right)\Big \} ,\label{eq:policy_improvement}
\end{equation}
where $\alpha_{t}>0$ is the stepsize. This step corresponds to one
iteration of Proximal Policy Optimization (PPO) \citep{schulman2017proximal}.
It can also be viewed as one mirror descent iteration, where the shifted
Q-function $\Qhat_{t}^{\lambda}(s,\cdot)-\lambda\log\pi_{t}(\cdot|s)$
plays the role of the gradient. The maximizer $\widehat{\pi}_{t+1}$
in equation~(\ref{eq:policy_improvement}) can be computed in closed
form as done in equation~(\ref{eq:policy_update}) in Algorithm~\ref{alg:Fictitious-Play-embedded}.
We then compute the new policy $\pi_{t+1}$ by mixing $\pihat_{t+1}$
with a small amount of uniform distribution, as done in equation~(\ref{eq:policy_mix}).
``Mixing in'' a uniform distribution is a standard technique to
prevent the policy from approaching the boundary of the probability
simplex and becoming degenerate. Doing so allows us to upper bound
a quantity of the form $\KL\left(p\,\Vert\,\pi_{t+1}(\cdot|s)\right)$
(cf.\ Lemma~\ref{lem:mix_diff_bound}), which otherwise may be infinite.
It also ensures that the KL divergence satisfies a Lipschitz condition
(cf.\ Lemma~\ref{lem:mix_KL_lipschitz}).

\paragraph{Mean-field Update.}

We next compute an updated (embedded) mean-field state $\mu_{t+1}$
as a weighted average of the current $\mu_{t}$ and the mean-field
state $\Gamma_{2}(\pi_{t+1},\mu_{t})$ induced by the new policy $\pi_{t+1}$,
namely, $\mu_{t+1}=(1-\beta_{t})\mu_{t}+\beta_{t}\cdot\Gamma_{2}(\pi_{t+1},\mu_{t}),$
where $\beta_{t}\in(0,1)$ is the stepsize. This update can be viewed
as a single step of the (soft) fixed point iteration for the equation
$\mu=\Gamma_{2}(\pi_{t+1},\mu)$. 

We remark that our algorithm is similar to the classical fictitious
play approach for finding NEs, where each agent plays a response to
the empirical average of its opponent's past behaviors. In our algorithm,
the representative agent views the population of all agents collectively
as an opponent. Expanding the recursion (\ref{eq:policy_update})
and ignoring the difference between $\widehat{\pi}_{t+1}$ and $\pi_{t+1}$,
we can write the policy $\pi_{t+1}$ as 
\[
\pi_{t+1}(\cdot|s)\propto\exp\biggl ({ \sum_{\tau=0}^{t} } w_{\tau}\Qhat_{\tau}^{\lambda}(s,\cdot)\biggr )
\]
for some positive weights $\{w_{\tau}\}$. Therefore, the representative
agent is playing a policy that responds to the (weighted) average
of all previous Q functions, which reflects the representative agent's
belief on the aggregate population policy. 

Also note that our algorithm only performs a single policy improvement
step to compute the updated policy $\pi_{t+1}$. It is unnecessary
to compute the exact optimal policy $\pi_{t+1}^{*}=\Gamma_{1}^{\lambda}(\mu_{t})$
under $\mu_{t}$ (which would require an inner loop for solving $\mdp_{\mu_{t}})$,
as $\mu_{t}$ is only an approximate anyway of the true NE mean-field
$\mu^{*}$. Our algorithm updates $\pi_{t}$ and $\mu_{t}$ simultaneously
within a single loop.

\section{Main Results}

In this section, we establish the theoretical guarantees on learning
the regularized NE $(\pi^{*},\mu^{*})$ of the MFG for our fictitious
play algorithm. To state our theorem, we first discuss several regularity
assumptions on the MFG model. Recall the definition~(\ref{eq:def_state_visitation_dist})
of the discounted state visitation distribution and let $\rho^{*}:=\rho_{\mu^{*}}^{\pi^{*}}\in\simplex(\mS)$
be the visitation distribution induced by the NE $(\pi^{*},\mu^{*})$.
We make use of the following distance metric between two policies
$\pi,\pi'\in\Pi$:
\begin{equation}
D(\pi,\pi'):=\E_{s\sim\rho^{*}}\left[\left\Vert \pi(\cdot|s)-\pi'(\cdot|s)\right\Vert _{1}\right].\label{eq:def_policy_distance}
\end{equation}

As in the classical MFG literature \citep{guo2020general,saldi2018markov},
we assume certain Lipschitz properties for the two mappings $\Gamma_{1}^{\lambda}:\muspace\to\Pi$
and $\Gamma_{2}:\Pi\times\muspace\to\muspace$ defined in Section~\ref{sec:regularization}.
The first assumption states that $\Gamma_{1}^{\lambda}(\mu)$ is Lipschitz
in the mean-embedded mean-field state $\mu$ with respect to the RKHS
norm.
\begin{assumption}
\label{assu:Lipz_Gamma_1}There exists a constant $d_{1}>0$, such
that for any $\mu,\mu'\in\muspace,$ it holds that
\[
D\left(\Gamma_{1}^{\lambda}(\mu),\Gamma_{1}^{\lambda}(\mu')\right)\le d_{1}\left\Vert \mu-\mu'\right\Vert _{\H}.
\]
\end{assumption}
The second assumption states that $\Gamma_{2}(\pi,\mu)$ is Lipschitz
in each of its arguments when the other argument is fixed. 
\begin{assumption}
\label{assu:Lipz_Gamma_2}There exist constants $d_{2}>0,d_{3}>0$
such that for any policies $\pi,\pi'\in\Pi$ and embedded mean-field
states $\mu,\mu'\in\muspace$, it holds that
\begin{align*}
\left\Vert \Gamma_{2}(\pi,\mu)-\Gamma_{2}(\pi',\mu)\right\Vert _{\H}   \le d_{2}D\left(\pi,\pi'\right), 
\quad \left\Vert \Gamma_{2}(\pi,\mu)-\Gamma_{2}(\pi,\mu')\right\Vert _{\H}   \le d_{3}\left\Vert \mu-\mu'\right\Vert _{\H}.
\end{align*}
\end{assumption}
Assumptions \ref{assu:Lipz_Gamma_1} and~\ref{assu:Lipz_Gamma_2}
immediately imply Lipschitzness of the composite mapping $\Lambda^{\lambda}:\muspace\to\muspace$,
which we recall is defined as $\Lambda^{\lambda}(\mu)=\Gamma_{2}\left(\Gamma_{1}^{\lambda}(\mu),\mu\right).$
The proof is provided in Appendix \ref{subsec:proof_Lipz_Lambda}.
\begin{lem}
\label{lem:Lipz_Lambda}Suppose Assumptions \ref{assu:Lipz_Gamma_1}
and~\ref{assu:Lipz_Gamma_2} hold. Then for each $\mu,\mu'\in\muspace,$
it holds that
\[
\left\Vert \Lambda^{\lambda}(\mu)-\Lambda^{\lambda}(\mu')\right\Vert _{\H}\leq(d_{1}d_{2}+d_{3})\left\Vert \mu-\mu'\right\Vert _{\H}.
\]
\end{lem}
We next impose an assumption on the boundedness of certain concentrability
coefficients. This type of assumption, standard in analysis of policy
optimization algorithms \citep{kakade2002approximately,shani2019adaptive,bhandari2019global,agarwal2020optimality},
allows one to define the policy optimization error in an average-case
sense with respect to appropriate distributions over the states. 
\begin{assumption}[Finite Concentrability Coefficients]
\label{assu:concentrability} There exist two constants $C_{\rho},\overline{C}_{\rho}>0$
such that for each $\mu\in\muspace,$ it holds that
\[
\Biggl \|  \frac{ \rho_{\mu}^{\pi_{\mu}^{\lambda,*} }}{ \rho^{*} }\Biggr \|  _{\infty}
	:=\sup_{s} \Biggl [  \frac{ \rho_{\mu}^{\pi_{\mu}^{\lambda,*}}(s) }{ \rho^{*}(s) } \Biggr ] 
	\le C_{\rho}
	\quad\text{and}\quad
	\Biggl \{ \E_{s\sim\rho_{\mu}^{\pi_{\mu}^{\lambda,*}}}\Bigg [\bigg  | \frac{ \rho^{*}(s) } { \rho_{\mu}^{\pi_{\mu}^{\lambda,*}}(s) } \bigg|^{2}\Bigg ]\Biggr \} ^{1/2}
	\le\overline{C}_{\rho}.
\]
\end{assumption}
Finally, our last assumption stipulates that the state visitation
distributions are smooth with respect to the (embedded) mean-field
states of the MFG. This assumption is analogous to those in the literature
on MDP and two-player games \citep{fei2020dynamic,radanovic2019learning},
which requires the visitation distributions to be smooth with respect
to the policy. 
\begin{assumption}
\label{assu:Lipz_visitation}There exists a constant $d_{0}>0$, such
that for any $\mu,\mu'\in\mathcal{M},$ it holds that
the discounted state visitation distributions induced by the corresponding optimal policy $\pi_{\mu}^{\lambda,*}$ for regularized $\mdp_{\mu}$ and $\pi_{\mu'}^{\lambda,*}$ for regularized
$\mdp_{\mu'}$ satisfy 
\[
\Big \Vert \rho_{\mu}^{\pi_{\mu}^{\lambda,*}}-\rho_{\mu'}^{\pi_{\mu'}^{\lambda,*}}\Big \Vert _{1}\leq d_{0}\left\Vert \mu-\mu'\right\Vert _{\H}.
\]
\end{assumption}
We now state our theoretical guarantees on the convergences of the
policy-population sequence $\{\pi_{t},\mu_{t}\}$ in Algorithm~\ref{alg:Fictitious-Play-embedded}
to the NE $\left\{ \pi^{*},\mu^{*}\right\} $. For the estimates of
the embedded mean-field states, it is natural to consider the distance
$\left\Vert \mu_{t}-\mu^{*}\right\Vert _{\H}$ in RKHS norm. For convergence
to NE policy $\mu^{*}$, recall that $\mu^{*}$ is the optimal policy
to $\mdp_{\mu^{*}}$, and each iteration of our algorithm involves
a single policy improvement step to compute $\pi_{t+1}$ rather than
solving $\mdp_{\mu_{t}}$ to its optimal policy $\pi_{t+1}^{*}:=\Gamma_{1}^{\lambda}(\mu_{t})$.
As such, we analyze the difference between these two policies in terms
of $D\left(\pi_{t+1},\pi_{t+1}^{*}\right)$, where the metric $D$
is defined in equation~(\ref{eq:def_policy_distance}). Also let
$\rho_{t}^{*}:=\rho_{\mu_{t}}^{\pi_{t+1}^{*}}$ denote the discounted
visitation distribution induced by the optimal policy $\pi_{t+1}^{*}$
of $\mdp_{\mu_{t}}.$\footnote{The subscript in $\rho_{t}^{*}$ emphasizes that $\rho_{t}^{*}$ only
depends on the mean-field state $\mu_{t}$ at time $t$ through $\pi_{t+1}^{*}=\Gamma_{1}^{\lambda}(\mu_{t})$.} With the above considerations in mind, we have the following theorem,
which is proved in Appendix~\ref{sec:proof_main_thm}. .
\begin{thm}
\label{thm:main}Suppose that Assumptions \ref{assu:RKSH_kernel}--\ref{assu:Lipz_visitation}
hold and $d_{1}d_{2}+d_{3}<1$ and that the error in the policy evaluation
step in Algorithm~\ref{alg:Fictitious-Play-embedded} satisfies 
\[
\E_{s\sim\rho_{t}^{*}}\Big[\bigl \Vert Q_{t}^{\lambda}(s,\cdot)-\Qhat_{t}^{\lambda}(s,\cdot)\big \Vert _{\infty}^{2}\Big ]\leq\varepsilon^{2},\qquad\forall t\in[T].
\]
With the choice of 
\[
\eta=c_{\eta}T^{-1},\qquad\alpha_{t}\equiv\alpha=c_{\alpha}T^{-2/5},\qquad\beta_{t}\equiv\beta=c_{\beta}T^{-4/5},
\]
for some universal constants $c_{\eta}>0$, $c_{\alpha}>0$ and $c_{\beta}>0$
in Algorithm \ref{alg:Fictitious-Play-embedded}, the resulting policy
and embedded mean-field state sequence $\{ (\pi_{t},\mu_{t})\} _{t=1}^{T}$
satisfy
\begin{align}
D\Big (\frac{1}{T} { \sum_{t=1}^{T}} \pi_{t},\frac{1}{T} { \sum_{t=1}^{T}}\pi_{t}^{*}\Big )\leq \frac{1}{T} { \sum_{t=1}^{T}} D (\pi_{t},\pi_{t}^{*} ) & \lesssim \frac{1}{\sqrt{\lambda}} \cdot \Big ( \sqrt{\log T} \cdot T^{- 1/5} +\sqrt{\varepsilon} \Big ),\label{eq:main_policy_bound}\\
\Big \Vert \frac{1}{T} { \sum_{t=1}^{T}} \mu_{t}-\mu^{*} \Big\Vert _{\H}\le \frac{1}{T} { \sum_{t=1}^{T}}  \Vert \mu_{t}-\mu^{*} \Vert_{\H} & \lesssim \frac{1}{\sqrt{\lambda}} \cdot \Big ( \sqrt{\log T} \cdot T^{- 1/5} +\sqrt{\varepsilon} \Big).\label{eq:main_mean_field_bound}
\end{align}
\end{thm}
Theorem~\ref{thm:main} bounds the distance between $\pi_{t}$ and
the optimal policy $\pi_{t}^{*}$ of $\mdp_{\mu_{t}^{*}}.$ By directly
measuring the distance between $\pi_{t}$ and the NE policy $\pi^{*}$,
we can define the notion of an $\delta$-approximate NE of the game. 
\begin{defn}
\label{def:approx_NE}For each $\delta>0$, a policy-population pair
$(\pi,\mu)$ is called an $\delta$-approximate (entropy-regularized) NE of the
MFG if 
\[
D(\pi,\pi^{*})\le\delta\quad\text{and}\quad\left\Vert \mu-\mu^{*}\right\Vert _{\H}\le\delta.
\]
\end{defn}

The following corollary of Theorem~\ref{thm:main} shows that after
$T$ iterations of our algorithm, the average policy-population pair
$ (\frac{1}{T}\sum_{t=1}^{T}\pi_{t},\frac{1}{T}\sum_{t=1}^{T}\mu_{t} )$
is an $\widetilde{\mathcal{O}}\left(T^{-1/5}\right)$-approximate
NE.
\begin{cor}
\label{cor:main}Under the assumptions of Theorem~\ref{thm:main},
we have
\[
D  \Big( \frac{1}{T} { \sum_{t=1}^{T}} \pi_{t},\pi^{*}\Big)+\Big \Vert \frac{1}{T} { \sum_{t=1}^{T}} \mu_{t}-\mu^{*}\Big \Vert _{\mH}\lesssim \frac{1}{\sqrt{\lambda}} \cdot \Big ( \sqrt{\log T} \cdot T^{- 1/5} +\sqrt{\varepsilon} \Big).
\]
\end{cor}
\noindent We prove this corollary in Appendix~\ref{sec:proof_cor_main}.

The above results require an $\ell_{2}$-error of $\varepsilon$ for
policy evaluation. A variety of algorithms have been shown to achieve
such a guarantees, including TD(0) and LSTD \citep{bhandari2018finite}. It is also worth emphasizing that the convergence rate to the regularized NE scales inverse proportionally with $\sqrt{\lambda},$ implying that convergence can be accelerated with a higher level of entropy regularization. On other hand, the approximation error of the regularized NE for the original unregularized NE scales proportionally with $\lambda$ (cf.~(\ref{eq:ER_optimality_gap})). Therefore, it is desirable to choose the regularization parameter $\lambda$ that balances the target accuracy level and convergence rate.

\subsection{Guarantees under Weaker Concentrability Assumption \label{sec:results_W}}

In this section, we show that the $\ell_{\infty}$ condition on concentrability
coefficient in Assumption~\ref{assu:concentrability} can be relaxed
to an $\ell_{2}$ condition of the form $\big\{\E\big[\big|\rho_{\mu}^{\pi_{\mu}^{\lambda,*}}(s)/\rho^{*}(s)\big|^{2}\big]\big\}^{1/2}\le C{}_{\rho}$, under which we can establish an $\tildeO(T^{-1/9})$ convergence
rate.

We now provided the details. Consider the following distance metric between two policies $\pi,\pi'\in\Pi$:
\begin{equation}
W(\pi,\pi'):=\sqrt{\E_{s\sim\rho^{*}}\left[\left\Vert \pi(\cdot|s)-\pi'(\cdot|s)\right\Vert _{1}^{2}\right]}.\label{eq:def_policy_distance_W}
\end{equation}
Similarly as before, we assume certain Lipschitz properties for the
two mappings $\Gamma_{1}^{\lambda}:\muspace\to\Pi$ and $\Gamma_{2}:\Pi\times\muspace\to\muspace$
defined in Section~\ref{sec:regularization}. In particular, we impose
the following two assumtpions, both stated in terms of the new distance
metric $W(\cdot,\cdot)$ defined in (\ref{eq:def_policy_distance_W})
above.
\begin{assumption}
	\label{assu:Lipz_Gamma_1_W}There exists a constant $d_{1}>0$, such
	that for any $\mu,\mu'\in\muspace,$ it holds that
	\[
	W\left(\Gamma_{1}^{\lambda}(\mu),\Gamma_{1}^{\lambda}(\mu')\right)\le d{}_{1}\left\Vert \mu-\mu'\right\Vert _{\H}.
	\]
\end{assumption}
\begin{assumption}
	\label{assu:Lipz_Gamma_2_W}There exist constants $d_{2}>0,d_{3}>0$
	such that for any policies $\pi,\pi'\in\Pi$ and embedded mean-field
	states $\mu,\mu'\in\muspace$, it holds that
	\begin{align*}
	\left\Vert \Gamma_{2}(\pi,\mu)-\Gamma_{2}(\pi',\mu)\right\Vert _{\H} & \le d_{2}W\left(\pi,\pi'\right),\\
	\left\Vert \Gamma_{2}(\pi,\mu)-\Gamma_{2}(\pi,\mu')\right\Vert _{\H} & \le d_{3}\left\Vert \mu-\mu'\right\Vert _{\H}.
	\end{align*}
\end{assumption}
Assumptions \ref{assu:Lipz_Gamma_1_W} and~\ref{assu:Lipz_Gamma_2_W}
immediately imply Lipschitzness of the composite mapping $\Lambda^{\lambda}:\muspace\to\muspace$,
which we recall is defined as $\Lambda^{\lambda}(\mu)=\Gamma_{2}\left(\Gamma_{1}^{\lambda}(\mu),\mu\right).$
\begin{lem}
	\label{lem:Lipz_Lambda-1}Suppose Assumptions \ref{assu:Lipz_Gamma_1_W}
	and~\ref{assu:Lipz_Gamma_2_W} hold. Then for each $\mu,\mu'\in\muspace,$
	it holds that
	\[
	\left\Vert \Lambda^{\lambda}(\mu)-\Lambda^{\lambda}(\mu')\right\Vert _{\H}\leq(d{}_{1}d_{2}+d_{3})\left\Vert \mu-\mu'\right\Vert _{\H}.
	\]
\end{lem}
We also consider the following relaxed, $\ell_{2}$-type assumption
on the concentrability coefficients. 
\begin{assumption}[Finite Concentrability Coefficients]
	\label{assu:concentrability_W} There exist two constants $C_{\rho},\overline{C}{}_{\rho}>0$
	such that for each $\mu\in\muspace,$ it holds that
	\[
	\left\{ \E_{s\sim\rho_{\mu}^{\pi_{\mu}^{\lambda,*}}}\left[\left|\frac{\rho_{\mu}^{\pi_{\mu}^{\lambda,*}}(s)}{\rho^{*}(s)}\right|^{2}\right]\right\} ^{1/2}\le C{}_{\rho}\qquad\text{and}\qquad\left\{ \E_{s\sim\rho_{\mu}^{\pi_{\mu}^{\lambda,*}}}\left[\left|\frac{\rho^{*}(s)}{\rho_{\mu}^{\pi_{\mu}^{\lambda,*}}(s)}\right|^{2}\right]\right\} ^{1/2}\le\overline{C}{}_{\rho}.
	\]
\end{assumption}

With the above assumptions and the distance metric $ W $, we can establish the following convergence result for Algorithm \ref{alg:Fictitious-Play-embedded}.

\begin{thm}
	\label{thm:main_W}Suppose that Assumptions \ref{assu:RKSH_kernel},
	\ref{assu:Lipz_visitation}, \ref{assu:Lipz_Gamma_1_W}, \ref{assu:Lipz_Gamma_2_W},
	and \ref{assu:concentrability_W} hold and $d_{1}d_{2}+d_{3}<1$ and
	that the error in the policy evaluation step in Algorithm~\ref{alg:Fictitious-Play-embedded}
	satisfies 
	\[
	\E_{s\sim\rho_{t}^{*}}\left[\left\Vert Q_{t}^{\lambda}(s,\cdot)-\Qhat_{t}^{\lambda}(s,\cdot)\right\Vert _{\infty}^{2}\right]\leq\varepsilon^{2},\qquad\forall t\in[T].
	\]
	With the choice of 
	\[
	\eta=c_{\eta}T^{-1},\qquad\alpha_{t}\equiv\alpha=c_{\alpha}T^{-4/9},\qquad\beta_{t}\equiv\beta=c_{\beta}T^{-8/9},
	\]
	for some universal constants $c_{\eta}>0$, $c_{\alpha}>0$ and $c_{\beta}>0$
	in Algorithm \ref{alg:Fictitious-Play-embedded}, the resulting policy
	and embedded mean-field state sequence $\left\{ (\pi_{t},\mu_{t})\right\} _{t=1}^{T}$
	satisfy
	\begin{align}
	W\left(\frac{1}{T}\sum_{t=1}^{T}\pi_{t},\frac{1}{T}\sum_{t=1}^{T}\pi_{t}^{*}\right)\le\frac{1}{T}\sum_{t=1}^{T}W\left(\pi_{t},\pi_{t}^{*}\right) & \lesssim\frac{1}{\lambda^{1/4}}\left(\frac{(\log T)^{1/4}}{T^{1/9}}+\varepsilon^{1/4}\right),\label{eq:main_policy_bound_W}\\
	\left\Vert \frac{1}{T}\sum_{t=1}^{T}\mu_{t}-\mu^{*}\right\Vert _{\H}\le\frac{1}{T}\sum_{t=1}^{T}\left\Vert \mu_{t}-\mu^{*}\right\Vert _{\H} & \lesssim\frac{1}{\lambda^{1/4}}\left(\frac{(\log T)^{1/4}}{T^{1/9}}+\varepsilon^{1/4}\right).\label{eq:main_mean_field_bound_W}
	\end{align}
\end{thm}
The following corollary of Theorem~\ref{thm:main_W} shows that after
$T$ iterations of our algorithm, the average policy-population pair
$\left(\frac{1}{T}\sum_{t=1}^{T}\pi_{t},\frac{1}{T}\sum_{t=1}^{T}\mu_{t}\right)$
is an $\widetilde{\mathcal{O}}\left(T^{-1/9}\right)$-approximate
NE. 
\begin{cor}
	\label{cor:main_W}Under the assumptions of Theorem~\ref{thm:main_W},
	we have
	\[
	W\left(\frac{1}{T}\sum_{t=1}^{T}\pi_{t},\pi^{*}\right)+\left\Vert \frac{1}{T}\sum_{t=1}^{T}\mu_{t}-\mu^{*}\right\Vert _{\mH}\lesssim\frac{1}{\lambda^{1/4}}\left(\frac{(\log T)^{1/4}}{T^{1/9}}+\varepsilon^{1/4}\right).
	\]
\end{cor}

We provide the proofs of Theorem~\ref{thm:main_W} and Corollary~\ref{cor:main_W} in Appendix~\ref{sec:proof_main_W}.

\section{Conclusion}\label{sec:conclusion}

In this paper, we develop a provably efficient fictitious play algorithm for stationary mean-field games. In comparison to the existing work that requires solving an MDP induced by a mean-field state within each iteration, our algorithm updates both the policy and the mean-field state simultaneously in each iteration. We prove that the  policy and mean-field state sequence generated by the proposed  algorithm converges to the Nash equilibrium of the MFG at a sublinear rate.

A number of directions are of interest for future research. An immediate step is to investigate whether the convergence rate can be improved. The $\tildeO(T^{-1/5})$ convergence rate we showed used constant step-sizes. It would be interesting to see if using time-varying step-sizes can attain a faster convergence rate. Another research direction worth pursuing is generalizing our approach for developing decentralized/distributed learning schemes.

\bibliographystyle{apalike}
\bibliography{../ICLR2021/rl_refs}

\appendix
\appendixpage

\section{Technical Lemmas \label{sec:tech_lemma}\protect 
}\begin{lem}
\label{lem:mix_diff_bound}Let $p^{*}$ and $p\in\simplex(\mA)$ and
$\widehat{p}=(1-\eta)p+\eta\frac{\boldsymbol{1}_{|\mA|}}{\left|\mA\right|}.$
Then 
\begin{align*}
\KL\left(p^{*}\Vert\widehat{p}\right) & \le\log\frac{\left|\mA\right|}{\eta},\\
\KL\left(p^{*}\Vert\widehat{p}\right)-\KL\left(p^{*}\Vert p\right) & \le2\eta.
\end{align*}
\end{lem}
\begin{proof}
By definition we have 
\begin{align*}
\KL\left(p^{*}\Vert\widehat{p}\right) & =\sum_{a\in\mA}p^{*}(a)\log\frac{p^{*}(a)}{\widehat{p}(a)}\\
 & =\sum_{a\in\mA}p^{*}(a)\log\frac{p^{*}(a)}{(1-\eta)p(a)+\frac{\eta}{\left|\mA\right|}}\\
 & \le\sum_{a\in\mA}p^{*}(a)\log\frac{1}{0+\frac{\eta}{\left|\mA\right|}}\\
 & =\log\frac{\left|\mA\right|}{\eta},
\end{align*}
thereby proving the first inequality. 

Note that 
\begin{equation}
\KL\left(p^{*}\Vert\widehat{p}\right)-\KL\left(p^{*}\Vert p\right)=\sum_{a\in\mA}p^{*}(a)\log\left(\frac{p(a)}{\widehat{p}(a)}\right).\label{eq:KL_gap}
\end{equation}
If $\frac{p(a)}{\widehat{p}(a)}\leq1$ for all $a\in\mA$ then we
have 
\[
\KL\left(p^{*}\Vert\widehat{p}\right)-\KL\left(p^{*}\Vert p\right)\leq0;
\]
otherwise, there exists $a'$ such that $p(a')\geq\widehat{p}(a')$
and we have 
\begin{align*}
\log\left(\frac{p(a')}{\widehat{p}(a')}\right) & =\log\left(\frac{p(a')}{(1-\eta)p(a')+\eta/|\mA|}\right)\\
 & \leq\log\left(\frac{p(a')}{(1-\eta)p(a')}\right)\\
 & \leq\frac{\eta}{1-\eta}\leq2\eta,
\end{align*}
where the third step follows from the fact that $\log(z)\leq z-1$
for all $z>0$ and the last step holds as $\eta\in[0,\frac{1}{2}]$.
Therefore, we have $\log\left(\frac{p(a')}{\widehat{p}(a')}\right)\leq2\eta$.
Applying Holder's inequality to (\ref{eq:KL_gap}) completes the proof.
\end{proof}
\begin{lem}
\label{lem:mix_KL_lipschitz}Let $x,y$ and $z\in\simplex(\mA)$.
If $x(a)\ge\alpha_{1}$, $y(a)\ge\alpha_{1}$ and $z(a)\ge\alpha_{2}$
for all $a\in\mA$, then
\[
\KL(x\Vert z)-\KL(y\Vert z)\le\left(1+\log\frac{1}{\min\left\{ \alpha_{1},\alpha_{2}\right\} }\right)\cdot\left\Vert x-y\right\Vert _{1}.
\]
\end{lem}
\begin{proof}
Under the lower bound assumption of the lemma, we have 
\[
\frac{\ddup\KL(x\Vert z)}{\ddup x(a)}=1+\log\frac{x(a)}{z(a)}\le1+\log\frac{1}{\alpha_{2}}
\]
and 
\[
-\frac{\ddup\KL(x\Vert z)}{\ddup x(a)}\le-1-\log\alpha_{1}.
\]
It follows that 
\[
\left\Vert \frac{\ddup\KL(x\Vert z)}{\ddup x(a)}\right\Vert _{\infty}\le\max\left\{ 1+\log\frac{1}{\alpha_{2}},-1-\log\alpha_{1}\right\} \le1+\log\frac{1}{\min\left\{ \alpha_{1},\alpha_{2}\right\} }.
\]
Hence the function $x\mapsto\KL(x\Vert z)$ is Lipschitz w.r.t. $\left\Vert \cdot\right\Vert _{1}$,
the dual norm of $\left\Vert \cdot\right\Vert _{\infty}.$
\end{proof}

\section{Proof of Theorem \ref{thm:main} \label{sec:proof_main_thm}}

In order to obtain an upper bound on the optimality gap 
\begin{equation}
\gapmu^{t}:=\left\Vert \mu_{t}-\mu^{*}\right\Vert _{\H},\label{eq:def_mu_error}
\end{equation}
where $\mu^{*}$ is the embedded mean-field state of the entropy regularized
NE, we also need to estimate the gap between $\pi_{t+1}$ and the
optimal solution to the entropy regularized $\mdp_{\mu_{t}}$. We
define 
\begin{equation}
\gappi^{t+1}:=\E_{s\sim\rho_{t}^{*}}\left[\KL\left(\pi_{t+1}^{*}(\cdot|s)\Vert\pi_{t+1}(\cdot|s)\right)\right]\label{eq:def_pi_error}
\end{equation}
to quantify the convergence of policy sequence. 

Before proceeding, we establish the following properties of entropy
regularized MDPs, which are central to the convergence analysis. 

\paragraph{Properties of Regularized MDP.}

The following lemma quantifies the performance difference between
two policies for a regularized MDP --- measured in terms of the expected
total reward --- through the Q-function and their KL-divergence.
The proof is provided in Appendix \ref{subsec:proof_performance_difference}. 
\begin{lem}[Performance Difference]
\label{lem:performance_difference} For each $\mu\in\muspace$ and
policies $\pi:\mS\rightarrow\simplex(\mA)$, it holds that 
\begin{align}
 & J_{\mu}^{\lambda}(\pi')-J_{\mu}^{\lambda}(\pi)+\frac{\lambda}{1-\gamma}\E_{s\sim\rho_{\mu}^{\pi'}}\left[\KL\left(\pi'(\cdot|s)\Vert\pi(\cdot|s)\right)\right]\nonumber \\
= & \frac{1}{1-\gamma}\E_{s\sim\rho_{\mu}^{\pi'}}\left[\left\langle Q_{\mu}^{\lambda,\pi}(s,\cdot)-\lambda\log\pi(\cdot|s),\pi'(\cdot|s)-\pi(\cdot|s)\right\rangle \right],\label{eq:performance_difference}
\end{align}
where $\rho_{\mu}^{\pi'}$ is the discounted state visitation distribution
induced by the policy $\pi'$ on $\mdp_{\mu}$.
\end{lem}
We can characterize the optimal policy $\pi_{\mu}^{\lambda,*}$ in
terms of the optimal Q-function $Q_{\mu}^{\lambda,*}$ as a Boltzmann
distribution of the form \citet{cen2020fast,nachum2017bridging}
\begin{equation}
\pi_{\mu}^{\lambda,*}(a|s)\propto\exp\left(\frac{Q_{\mu}^{\lambda,*}(s,a)}{\lambda}\right).\label{eq:optimal_ER_pi}
\end{equation}
For the setting where the reward function is bounded, we then can
obtain a lower bound on $\pi_{\mu}^{\lambda,*}$, as stated in the
following lemma. The proof is provided in Appendix \ref{subsec:proof_optimal_ER}
\begin{lem}
\label{lem:optimal_ER_MDP}Suppose that there exists a constant $R_{\max}>0$
such that $0\leq\sup_{(s,a,\mu)\in\mS\times\mA\times\muspace}r(s,a,\mu)\leq R_{\max}$.
For each $\mu\in\muspace$, and each policy $\pi:\mS\rightarrow\simplex(\mA)$,
we have 
\[
\left\Vert Q_{\mu}^{\lambda,\pi}\right\Vert _{\infty}\le Q_{\max}:=\frac{R_{\max}+\gamma\lambda\log\left|\mA\right|}{1-\gamma}.
\]
Also, the optimal policy $\pi_{\mu}^{\lambda,*}$ for the regularized
$\mdp_{\mu}$ satisfies 
\[
\pi_{\mu}^{\lambda,*}(a|s)\geq\frac{1}{e^{Q_{\max}/\lambda}|\mA|},\forall s\in\mS,a\in\mA.
\]
\end{lem}

\paragraph{Convergence Analysis. }

We now move to the convergence analysis. For clarity of exposition,
we use $\E_{\rho}\left[\left\Vert \pi-\pi'\right\Vert _{1}\right]$
as shorthand for $\E_{s\sim\rho}\left[\left\Vert \pi(\cdot|s)-\pi'(\cdot|s)\right\Vert _{1}\right]$,
where $\rho\in\simplex(\mS)$; we also use $\E_{\rho}\left[\KL\left(\pi\Vert\pi'\right)\right]$
as shorthand for $\E_{s\sim\rho}\left[\KL\left(\pi(\cdot|s)\Vert\pi'(\cdot|s)\right)\right]$.
We recall that the step sizes are chosen as 
\[
\alpha_{t}\equiv\alpha=c_{\alpha}T^{-2/5},\qquad\beta_{t}\equiv\beta=c_{\beta}T^{-4/5},
\]
where the parameters $c_{\alpha}$ and $c_{\beta}$ satisfy that:
\textbf{}
\begin{equation}
c_{\alpha}T^{-2/5}\lambda<1,\qquad c_{\beta}T^{-4/5}\overline{d}<1.\label{eq:beta_parameter}
\end{equation}
Here $\overline{d}:=1-d_{1}d_{2}-d_{3}>0$, where $d_{1}$ appears
in Assumption \ref{assu:Lipz_Gamma_1}, and $d_{2}$, $d_{3}$ appear
in Assumption~\ref{assu:Lipz_Gamma_2}.

\paragraph{Step 1: Convergence of Policy. }

To analyze the convergence of the optimality gap $\gapmu^{t+1}=\left\Vert \mu_{t+1}-\mu^{*}\right\Vert _{\H}$,
we first characterize the convergence behavior of the policy sequence
$\{\pi_{t}\}_{t\geq0}$. In particular, we establish a recursive relationship
between $\gappi^{t+1}=\E_{s\sim\rho_{t}^{*}}\left[\KL\left(\pi_{t+1}^{*}(\cdot|s)\Vert\pi_{t+1}(\cdot|s)\right)\right]$
and $\sigma_{\pi}^{t}$, as stated in the following lemma. The proof
is provided in Section \ref{subsec:proof_KL_recursion}. 
\begin{lem}
\label{lem:KL_recursion}Under the setting of Theorem \ref{thm:main},
for each $t\geq1$, we have 
\begin{equation}
\gappi^{t+1}\leq(1-\lambda\alpha_{t})\gappi^{t}+(1-\lambda\alpha_{t})\left(d_{0}\log\frac{\left|\mA\right|}{\eta}+\kappa C_{\rho}d_{1}\right)\left\Vert \mu_{t-1}-\mu_{t}\right\Vert _{\H}+2\varepsilon\alpha_{t}+\frac{Q_{\max}^{2}}{2}\alpha_{t}^{2}+2\eta,\label{eq:KL_recur_1}
\end{equation}
where $\kappa=\frac{4}{1-\gamma}\log\frac{\left|\mA\right|}{\eta}+\frac{2R_{\max}}{\lambda(1-\gamma)}.$
\end{lem}
Recall that $\mu_{t}=(1-\beta_{t-1})\mu_{t-1}+\beta_{t-1}\cdot\Gamma_{2}(\pi_{t},\mu_{t-1})$.\textbf{
}Under Assumption \ref{assu:RKSH_kernel}, we have
\begin{align}
\left\Vert \mu_{t-1}-\mu_{t}\right\Vert _{\mH} & =\beta_{t-1}\left\Vert \mu_{t-1}-\Gamma_{2}(\pi_{t},\mu_{t-1})\right\Vert _{\mH}\leq2\beta_{t-1}.\label{eq:mu_bounded}
\end{align}
Lemma \ref{lem:KL_recursion} implies that 
\begin{align}
\gappi^{t+1} & \leq(1-\lambda\alpha_{t})\gappi^{t}+(1-\lambda\alpha_{t})\overline{C}_{1}\beta_{t-1}+2\varepsilon\alpha_{t}+\frac{Q_{\max}^{2}}{2}\alpha_{t}^{2}+2\eta,\label{eq:KL_recur_2}
\end{align}
where we define 
\[
\overline{C}_{1}:=2\left(d_{0}\log\frac{\left|\mA\right|}{\eta}+\kappa C_{\rho}d_{1}\right).
\]

With $\alpha_{t}\equiv\alpha$, $\beta_{t}\equiv\beta$, from Equation
(\ref{eq:KL_recur_2}) we have that 
\begin{align}
\gappi^{t} & \text{\ensuremath{\leq}}\frac{1}{\lambda\alpha}\left(\gappi^{t}-\gappi^{t+1}\right)+\left(\frac{1}{\lambda\alpha}-1\right)\overline{C}_{1}\beta+\frac{2\varepsilon}{\lambda}+\frac{Q_{\max}^{2}}{2\lambda}\alpha+\frac{2\eta}{\lambda\alpha}.\label{eq:KL_recur_3}
\end{align}
Summing over $\ell=0,2,\ldots T-1$ on both sides of (\ref{eq:KL_recur_3})
and dividing by $t$ gives 
\begin{align}
\frac{1}{T}\sum_{t=0}^{T-1}\gappi^{t} & \leq\frac{1}{T\lambda\alpha}\left(\gappi^{0}-\gappi^{T}\right)+\left(\frac{1}{\lambda\alpha}-1\right)\overline{C}_{1}\beta+\frac{2\varepsilon}{\lambda}+\frac{Q_{\max}^{2}}{2\lambda}\alpha+\frac{2\eta}{\lambda\alpha}\nonumber \\
 & \leq\frac{1}{T\lambda\alpha}\gappi^{0}+\frac{\overline{C}_{1}\beta}{\lambda\alpha}+\frac{2\varepsilon}{\lambda}+\frac{Q_{\max}^{2}}{2\lambda}\alpha+\frac{2\eta}{\lambda\alpha}.\label{eq:KL_bound-1}
\end{align}
When choosing $\alpha=\bigO(T^{-2/5})$, $\beta=\bigO(T^{-4/5})$
and $\eta=\bigO(T^{-1})$, we have $\overline{C}_{1}=\bigO(\log T)$.
Therefore, we obtain
\begin{equation}
\frac{1}{T}\sum_{t=0}^{T-1}\gappi^{t}\lesssim\frac{\log T}{\lambda T^{2/5}}+\frac{2\varepsilon}{\lambda}.\label{eq:KL_convergence}
\end{equation}
If we let $\mathbb{\mathsf{T}}$ be a random number sampled uniformly
from $\{1,\ldots,T\},$ then the above equation can be written equivalently
as 
\begin{equation}
\E_{\mathsf{T}}\left[\sigma_{\pi}^{\mathsf{T}}\right]\lesssim\frac{\log T}{\lambda T^{2/5}}+\frac{2\varepsilon}{\lambda}.\label{eq:policy_bound}
\end{equation}

\paragraph{Step 2: Convergence of Mean-field Embedding. }

We now proceed to characterize the optimality gap for the embedded
mean-field state. We obtain the following upper bound on the optimality
gap $\gapmu^{t+1}=\left\Vert \mu_{t+1}-\mu^{*}\right\Vert _{\H}$.
The proof is provided in Section \ref{subsec:proof_mu_recursion}.
\begin{lem}
\label{lem:mu_recursion}Under the setting of Theorem \ref{thm:main},
for each $t\geq0$, we have 
\begin{align*}
\gapmu^{t+1}\text{\ensuremath{\leq}}\left(1-\beta_{t}\overline{d}\right)\gapmu^{t}+d_{2}\overline{C}_{\rho}\beta_{t}\sqrt{\gappi^{t+1}},
\end{align*}
where $\overline{d}=1-d_{1}d_{2}-d_{3}>0$.
\end{lem}

Lemma \ref{lem:mu_recursion} implies that 
\begin{equation}
\gapmu^{t}\le\frac{1}{\overline{d}\beta_{t}}\left(\gapmu^{t}-\gapmu^{t+1}\right)+\frac{d_{2}\overline{C}_{\rho}}{\overline{d}}\sqrt{\gappi^{t+1}}.\label{eq:mu_recur-1}
\end{equation}
With $\beta_{t}\equiv\beta=\bigO(T^{-4/5})$, averaging equation~(\ref{eq:mu_recur-1})
over iteration $t=0,\ldots,T-1$, we obtain 
\begin{align*}
\frac{1}{T}\sum_{t=0}^{T-1}\gapmu^{t} & \leq\frac{1}{\overline{d}\beta T}\left(\gapmu^{0}-\gapmu^{T}\right)+\frac{d_{2}\overline{C}_{\rho}}{\overline{d}T}\sum_{t=0}^{T-1}\sqrt{\gappi^{t+1}}\\
 & \leq\frac{\gapmu^{0}}{\overline{d}\beta T}+\frac{d_{2}\overline{C}_{\rho}}{\overline{d}T}\sum_{t=0}^{T-1}\sqrt{\gappi^{t+1}}\\
 & \leq\frac{\gapmu^{0}}{\overline{d}\beta T}+\frac{d_{2}\overline{C}_{\rho}}{\overline{d}}\sqrt{\frac{1}{T}\sum_{t=0}^{T-1}\gappi^{t+1}},
\end{align*}
where the last inequality follows from Cauchy-Schwarz inequality.

From Eq. (\ref{eq:KL_convergence}), we have
\begin{align*}
\frac{1}{T}\sum_{t=0}^{T-1}\gapmu^{t} & \lesssim\frac{\gapmu^{0}}{\overline{d}}T^{-1/5}+\frac{d_{2}\overline{C}_{\rho}}{\overline{d}}\sqrt{\frac{\log T}{\lambda T^{2/5}}+\frac{2\varepsilon}{\lambda}}\\
 & \lesssim\sqrt{\frac{\log T}{\lambda T^{2/5}}+\frac{2\varepsilon}{\lambda}}\\
 & \lesssim\frac{1}{\sqrt{\lambda}}\left(\frac{\sqrt{\log T}}{T^{1/5}}+\sqrt{\varepsilon}\right).
\end{align*}
This equation, together with Jensen's inequality, proves equation~(\ref{eq:main_mean_field_bound})
in Theorem~\ref{thm:main}. 

Turning to equation~(\ref{eq:main_policy_bound}) in Theorem~\ref{thm:main},
we have 
\begin{align*}
\frac{1}{T}\sum_{t=1}^{T}D\left(\pi_{t},\pi_{t}^{*}\right) & =\E_{\mathsf{T}}\left[D\left(\pi_{\mathsf{T}},\pi_{\mathsf{T}}^{*}\right)\right]\\
 & =\E_{\mathsf{T}}\E_{s\sim\rho^{*}}\left[\left\Vert \pi_{\mathsf{T}}^{*}(\cdot|s)-\pi_{\mathsf{T}}(\cdot|s)\right\Vert _{1}\right]\\
 & =\E_{\mathsf{T}}\E_{s\sim\rho_{\mathsf{T}-1}^{*}}\left[\frac{\rho^{*}(s)}{\rho_{\mathsf{T}-1}^{*}(s)}\left\Vert \pi_{\mathsf{T}}^{*}(\cdot|s)-\pi_{\mathsf{T}}(\cdot|s)\right\Vert _{1}\right]\\
 & \overset{(i)}{\le}\sqrt{\E_{\mathsf{T}}\E_{s\sim\rho_{\mathsf{T}-1}^{*}}\left[\left|\frac{\rho^{*}(s)}{\rho_{\mathsf{T}-1}^{*}(s)}\right|^{2}\right]\cdot\E_{\mathsf{T}}\E_{s\sim\rho_{\mathsf{T}-1}^{*}}\left[\left\Vert \pi_{\mathsf{T}}^{*}(\cdot|s)-\pi_{\mathsf{T}}(\cdot|s)\right\Vert _{1}^{2}\right]}\\
 & \overset{(ii)}{\le}\sqrt{\overline{C}_{\rho}^{2}\cdot\E_{\mathsf{T}}\E_{s\sim\rho_{\mathsf{T}-1}^{*}}\left[2\KL\left(\pi_{\mathsf{T}}^{*}(\cdot|s)\Vert\pi_{\mathsf{T}}(\cdot|s)\right)\right]}\\
 & =\sqrt{\overline{C}_{\rho}^{2}\cdot2\E_{\mathsf{T}}\left[\sigma_{\pi}^{\mathsf{T}}\right]}\\
 & \overset{(iii)}{\lesssim}\frac{1}{\sqrt{\lambda}}\left(\frac{\sqrt{\log T}}{T^{1/5}}+\sqrt{\varepsilon}\right),
\end{align*}
where step $(i)$ follows from Cauchy-Schwarz inequality, step $(ii)$
follows from Assumption~\ref{assu:concentrability} and Pinsker's
inequality, and step $(iii)$ follows from the bound in equation~(\ref{eq:policy_bound}).
The above equation, together with Jensen's inequality, proves equation~(\ref{eq:main_policy_bound}).
We have completed the proof of Theorem~\ref{thm:main}.

\subsection{Proof of Lemma \ref{lem:KL_recursion} \label{subsec:proof_KL_recursion}}

The following lemma characterizes this policy improvement step. The
proof is provided in Section \ref{subsec:proof_one_step_descent}. 
\begin{lem}
\label{lem:one_step_MD} For any distributions $p^{*},p\in\simplex(\mA),$state
$s\in\mS$ and function $G:\mS\times\mA\rightarrow\R$, it holds for
$p'\in\simplex(\mA)$ with $p'(\cdot)\propto p(\cdot)\cdot\exp\left[\alpha G(s,\cdot)\right]$
that 
\[
\KL\left(p^{*}\Vert p'\right)\leq\KL\left(p^{*}\Vert p\right)-\alpha\left\langle G(s,\cdot),p^{*}-p\right\rangle +\alpha^{2}\left\Vert G(s,\cdot)\right\Vert _{\infty}^{2}/2.
\]
\end{lem}
Recall that 
\begin{align*}
\widehat{\pi}{}_{t+1}(\cdot|s) & \propto\pi_{t}(\cdot|s)\cdot\exp\left[\alpha_{t}\left(\Qhat_{t}^{\lambda}(s,\cdot)-\lambda\log\pi_{t}(\cdot|s)\right)\right].
\end{align*}
Lemma \ref{lem:one_step_MD} implies that for each $s\in\mS,$ we
have 
\begin{align*}
 & \KL\left(\pi_{t+1}^{*}(\cdot|s)\Vert\widehat{\pi}_{t+1}(\cdot|s)\right)\\
\le & \KL\left(\pi_{t+1}^{*}(\cdot|s)\Vert\pi_{t}(\cdot|s)\right)-\alpha_{t}\left\langle \Qhat_{t}^{\lambda}(s,\cdot)-\lambda\log\pi_{t}(\cdot|s),\pi_{t+1}^{*}(\cdot|s)-\pi_{t}(\cdot|s)\right\rangle +\left\Vert \Qhat_{t}^{\lambda}\right\Vert _{\infty}^{2}\alpha_{t}^{2}/2\\
= & \KL\left(\pi_{t+1}^{*}(\cdot|s)\Vert\pi_{t}(\cdot|s)\right)-\alpha_{t}\left\langle Q_{t}^{\lambda}(s,\cdot)-\lambda\log\pi_{t}(\cdot|s),\pi_{t+1}^{*}(\cdot|s)-\pi_{t}(\cdot|s)\right\rangle \\
 & +\alpha_{t}\left\langle Q_{t}^{\lambda}(s,\cdot)-\Qhat_{t}^{\lambda}(s,\cdot),\pi_{t+1}^{*}(\cdot|s)-\pi_{t}(\cdot|s)\right\rangle +\left\Vert \Qhat_{t}^{\lambda}\right\Vert _{\infty}^{2}\alpha_{t}^{2}/2\\
\leq & \KL\left(\pi_{t+1}^{*}(\cdot|s)\Vert\pi_{t}(\cdot|s)\right)-\alpha_{t}\left\langle Q_{t}^{\lambda}(s,\cdot)-\lambda\log\pi_{t}(\cdot|s),\pi_{t+1}^{*}(\cdot|s)-\pi_{t}(\cdot|s)\right\rangle \\
 & +2\alpha_{t}\left\Vert Q_{t}^{\lambda}(s,\cdot)-\Qhat_{t}^{\lambda}(s,\cdot)\right\Vert _{\infty}+\left\Vert \Qhat_{t}^{\lambda}\right\Vert _{\infty}^{2}\alpha_{t}^{2}/2.
\end{align*}
Recall that  
$
\pi_{t+1}(\cdot|s)=(1-\eta)\widehat{\pi}_{t+1}(\cdot|s)+\frac{\eta}{|\mA|}\boldsymbol{1}_{|\mA|}. 
$
Lemma \ref{lem:mix_diff_bound} implies that 
\begin{align}
 & \KL\left(\pi_{t+1}^{*}(\cdot|s)\Vert\pi_{t+1}(\cdot|s)\right)\nonumber \\
\leq & \KL\left(\pi_{t+1}^{*}(\cdot|s)\Vert\widehat{\pi}_{t+1}(\cdot|s)\right)+2\eta.\\
\leq & \KL\left(\pi_{t+1}^{*}(\cdot|s)\Vert\pi_{t}(\cdot|s)\right)-\alpha_{t}\left\langle Q_{t}^{\lambda}(s,\cdot)-\lambda\log\pi_{t}(\cdot|s),\pi_{t+1}^{*}(\cdot|s)-\pi_{t}(\cdot|s)\right\rangle \nonumber \\
 & \qquad+\underbrace{2\alpha_{t}\left\Vert Q_{t}^{\lambda}(s,\cdot)-\Qhat_{t}^{\lambda}(s,\cdot)\right\Vert _{\infty}+\left\Vert \Qhat_{t}^{\lambda}\right\Vert _{\infty}^{2}\alpha_{t}^{2}/2+2\eta}_{Y_{t}(s)}.\label{eq:KL_recursion}
\end{align}

Taking expectation over $\rho_{t}^{*}$ on both sides of (\ref{eq:KL_recursion})
yields 
\begin{align}
 & \E_{\rho_{t}^{*}}\left[\KL\left(\pi_{t+1}^{*}\Vert\pi_{t+1}\right)\right]\nonumber \\
\le & \E_{\rho_{t}^{*}}\left[\KL\left(\pi_{t+1}^{*}\Vert\pi_{t}\right)\right]-\alpha_{t}\E_{s\sim\rho_{t}^{*}}\left[\left\langle Q_{t}^{\lambda}(s,\cdot)-\lambda\log\pi_{t}(\cdot|s),\pi_{t+1}^{*}(\cdot|s)-\pi_{t}(\cdot|s)\right\rangle \right]+\E_{s\sim\rho_{t}^{*}}\left[Y_{t}(s)\right]\nonumber \\
\stackrel{(a)}{=} & \E_{\rho_{t}^{*}}\left[\KL\left(\pi_{t+1}^{*}\Vert\pi_{t}\right)\right]-(1-\gamma)\alpha_{t}\left[J_{\mu_{t}}^{\lambda}(\pi_{t+1}^{*})-J_{\mu_{t}}^{\lambda}(\pi_{t})\right]-\alpha_{t}\lambda\E_{\rho_{t}^{*}}\left[\KL\left(\pi_{t+1}^{*}\Vert\pi_{t}\right)\right]+\E_{s\sim\rho_{t}^{*}}\left[Y_{t}(s)\right]\nonumber \\
\stackrel{(b)}{\leq} & (1-\alpha_{t}\lambda)\E_{\rho_{t}^{*}}\left[\KL\left(\pi_{t+1}^{*}\Vert\pi_{t}\right)\right]+\E_{s\sim\rho_{t}^{*}}\left[Y_{t}(s)\right]\nonumber \\
\stackrel{(c)}{\leq} & (1-\alpha_{t}\lambda)\underbrace{\E_{\rho_{t}^{*}}\left[\KL\left(\pi_{t}^{*}\Vert\pi_{t}\right)\right]}_{B_{1}}+(1-\alpha_{t}\lambda)\underbrace{\left|\E_{\rho_{t}^{*}}\left[\KL\left(\pi_{t+1}^{*}\Vert\pi_{t}\right)-\KL\left(\pi_{t}^{*}\Vert\pi_{t}\right)\right]\right|}_{B_{2}}+\E_{s\sim\rho_{t}^{*}}\left[Y_{t}(s)\right],\label{eq:KL_t_bound}
\end{align}
where step (a) follows from Lemma \ref{lem:performance_difference};
step (b) follows from the fact that $J_{\mu_{t}}^{\lambda}(\pi_{t})\leq J_{\mu_{t}}^{\lambda}(\pi_{t+1}^{*})$,
as $\pi_{t+1}^{*}=\Gamma_{1}^{\lambda}(\mu_{t})$ is the optimal policy
for the regularized $\mdp_{\mu_{t}}$; and step (c) holds due to triangle
inequality. 

Next we bound the first and second terms on the RHS of (\ref{eq:KL_t_bound})
separately. 
\begin{itemize}
\item For the second term $B_{2}$: Note that $\pi_{t+1}^{*}$ and $\pi_{t}^{*}$
are the optimal policy for the regularized $\mdp_{\mu_{t}}$ and $\mdp_{\mu_{t-1}},$ respectively.
Define 
\[
\tau:=\frac{1}{|\mA|}\exp\left(-\frac{R_{\max}+\gamma\lambda\log\left|\mA\right|}{\lambda(1-\gamma)}\right).
\]
By Lemma \ref{lem:optimal_ER_MDP}, for all $(s,a)\in\mS\times\mA$,
we have
\[
\pi_{t+1}^{*}(a|s)\geq\tau,\text{ and }\pi_{t}^{*}(a|s)\geq\tau.
\]
Applying Lemma \ref{lem:mix_KL_lipschitz} yields 
\begin{align}
B_{2} & \leq\kappa\E_{s\sim\rho_{t}^{*}}\left[\left\Vert \pi_{t}^{*}(\cdot|s)-\pi_{t+1}^{*}(\cdot|s)\right\Vert _{1}\right]\nonumber \\
 & =\kappa\E_{s\sim\rho^{*}}\left[\frac{\rho_{t}^{*}(s)}{\rho^{*}(s)}\cdot\left\Vert \pi_{t}^{*}(\cdot|s)-\pi_{t+1}^{*}(\cdot|s)\right\Vert _{1}\right]\nonumber \\
 & \leq\kappa C_{\rho}\E_{s\sim\rho^{*}}\left[\left\Vert \pi_{t}^{*}(\cdot|s)-\pi_{t+1}^{*}(\cdot|s)\right\Vert _{1}\right] &  & \text{Assumption \ref{assu:concentrability}}\nonumber \\
 & =\kappa C_{\rho}D\left(\Gamma_{1}^{\lambda}(\mu_{t-1}),\Gamma_{1}^{\lambda}(\mu_{t})\right)\nonumber \\
 & \leq\kappa C_{\rho}d_{1}\left\Vert \mu_{t-1}-\mu_{t}\right\Vert _{\H}, &  & \text{Assumption \eqref{assu:Lipz_Gamma_1}}\label{eq:bound_B2}
\end{align}
where 
\begin{align*}
\kappa & :=1+\log\frac{1}{\min\left\{ \tau,\frac{\eta}{|\mA|}\right\} }\\
 & \le2\max\left\{ \log\frac{\left|\mA\right|}{\eta},\frac{2}{1-\gamma}\log\left|\mA\right|+\frac{R_{\max}}{\lambda(1-\gamma)}\right\} \\
 & \le\frac{4}{1-\gamma}\log\frac{\left|\mA\right|}{\eta}+\frac{2R_{\max}}{\lambda(1-\gamma)} \\
 & =\frac{4}{1-\gamma}\KLmax+\frac{2R_{\max}}{\lambda(1-\gamma)}.
\end{align*}
 
\item For the first term $B_{1}$: We have 
\begin{align}
B_{1} & =\E_{\rho_{t-1}^{*}}\left[\KL\left(\pi_{t}^{*}\Vert\pi_{t}\right)\right]+\left(\E_{\rho_{t}^{*}}-\E_{\rho_{t-1}^{*}}\right)\left[\KL\left(\pi_{t}^{*}\Vert\pi_{t}\right)\right]\nonumber \\
 & =\E_{\rho_{t-1}^{*}}\left[\KL\left(\pi_{t}^{*}\Vert\pi_{t}\right)\right]+\E_{s\sim\rho^{*}}\left[\frac{\rho_{t}^{*}(s)-\rho_{t-1}^{*}(s)}{\rho^{*}(s)}\KL\left(\pi_{t}^{*}(\cdot|s)\Vert\pi_{t}(\cdot|s)\right)\right]\nonumber \\
 & \stackrel{(a)}{\leq}\E_{\rho_{t-1}^{*}}\left[\KL\left(\pi_{t}^{*}\Vert\pi_{t}\right)\right]+\E_{s\sim\rho^{*}}\left[\frac{\left|\rho_{t}^{*}(s)-\rho_{t-1}^{*}(s)\right|}{\rho^{*}(s)}\right]\cdot\KLmax,\nonumber \\
 & \stackrel{(b)}{\leq}\E_{\rho_{t-1}^{*}}\left[\KL\left(\pi_{t}^{*}\Vert\pi_{t}\right)\right]+\KLmax\cdot d_{0}\left\Vert \mu_{t}-\mu_{t-1}\right\Vert _{\H}\label{eq:bound_B1}
\end{align}
where step (a) uses the fact that $\grave{\KL\left(\pi_{t}^{*}(\cdot|s)\Vert\pi_{t}(\cdot|s)\right)}\le\KLmax:=\log\frac{|\mA|}{\eta}$
(cf. Lemma \ref{lem:mix_diff_bound}) and step (b) follows from
Assumption \ref{assu:Lipz_visitation}.
\end{itemize}
Combining (\ref{eq:KL_t_bound}), (\ref{eq:bound_B2}) and (\ref{eq:bound_B1}),
we have 
\begin{align}
 & \E_{\rho_{t}^{*}}\left[\KL\left(\pi_{t+1}^{*}\Vert\pi_{t+1}\right)\right]\nonumber \\
\leq & (1-\lambda\alpha_{t})\E_{\rho_{t-1}^{*}}\left[\KL\left(\pi_{t}^{*}\Vert\pi_{t}\right)\right]\nonumber \\
 & +(1-\lambda\alpha_{t})d_{0}\cdot\KLmax\left\Vert \mu_{t}-\mu_{t-1}\right\Vert _{\H}+(1-\lambda\alpha_{t})\kappa C_{\rho}d_{1}\left\Vert \mu_{t-1}-\mu_{t}\right\Vert _{\H}+\E_{s\sim\rho_{t}^{*}}\left[Y_{t}(s)\right]\nonumber \\
= & (1-\lambda\alpha_{t})\E_{\rho_{t-1}^{*}}\left[\KL\left(\pi_{t}^{*}\Vert\pi_{t}\right)\right] \nonumber \\
& +(1-\lambda\alpha_{t})\left(d_{0}\cdot\KLmax+\kappa C_{\rho}d_{1}\right)\left\Vert \mu_{t-1}-\mu_{t}\right\Vert _{\H}+\E_{s\sim\rho_{t}^{*}}\left[Y_{t}(s)\right].\label{eq:E_KL_recursion}
\end{align}
Note that 
\begin{align*}
\E_{s\sim\rho_{t}^{*}}\left[Y_{t}(s)\right] & =2\alpha_{t}\E_{s\sim\rho_{t}^{*}}\left[\left\Vert Q_{t}^{\lambda}(s,\cdot)-\Qhat_{t}^{\lambda}(s,\cdot)\right\Vert _{\infty}\right]+\frac{\left\Vert \Qhat_{t}^{\lambda}\right\Vert _{\infty}^{2}}{2}\alpha_{t}^{2}+2\eta\\
 & \le2\alpha_{t}\sqrt{\E_{s\sim\rho_{t}^{*}}\left[\left\Vert Q_{t}^{\lambda}(s,\cdot)-\Qhat_{t}^{\lambda}(s,\cdot)\right\Vert _{\infty}^{2}\right]}+\frac{\left\Vert \Qhat_{t}^{\lambda}\right\Vert _{\infty}^{2}}{2}\alpha_{t}^{2}+2\eta\\
 & \le2\varepsilon\alpha_{t}+\frac{Q_{\max}^{2}}{2}\alpha_{t}^{2}+2\eta,
\end{align*}
where the last step holds by the assumption on the policy evaluation
error and the fact that $\Qhat_{t-1}^{\lambda}:\mS\times\mA\to[0,Q_{\max}]$
satisfies $\left\Vert \Qhat_{t-1}^{\lambda}\right\Vert _{\infty}\le Q_{\max}$
by definition. Combining the last two display equations proves the
lemma.

\subsection{Proof of Lemma \ref{lem:mu_recursion} \label{subsec:proof_mu_recursion}}
\begin{proof}
According to the update rule (\ref{eq:mean_field_update}) for the
embedded mean-field state, we have
\begin{align}
 & \left\Vert \mu_{t+1}-\mu^{*}\right\Vert _{\H}\nonumber \\
= & \left\Vert (1-\beta_{t})\mu_{t}+\beta_{t}\Gamma_{2}(\pi_{t+1},\mu_{t})-\mu^{*}\right\Vert _{\H}\nonumber \\
= & \left\Vert (1-\beta_{t})\left(\mu_{t}-\mu^{*}\right)+\beta_{t}\left(\Gamma_{2}\left(\Gamma_{1}^{\lambda}(\mu_{t}),\mu_{t}\right)-\mu^{*}\right)-\beta_{t}\left[\Gamma_{2}\left(\Gamma_{1}^{\lambda}(\mu_{t}),\mu_{t}\right)-\Gamma_{2}(\pi_{t+1},\mu_{t})\right]\right\Vert _{\H}\nonumber \\
\leq & (1-\beta_{t})\left\Vert \left(\mu_{t}-\mu^{*}\right)\right\Vert _{\mH}+\beta_{t}\left\Vert \Gamma_{2}\left(\Gamma_{1}^{\lambda}(\mu_{t}),\mu_{t}\right)-\mu^{*}\right\Vert _{\mH}\nonumber \\
 & \qquad+\beta_{t}\left\Vert \Gamma_{2}\left(\Gamma_{1}^{\lambda}(\mu_{t}),\mu_{t}\right)-\Gamma_{2}(\pi_{t+1},\mu_{t})\right\Vert _{\mH}\nonumber \\
\stackrel{(i)}{=} & (1-\beta_{t})\left\Vert \mu_{t}-\mu^{*}\right\Vert _{\H}+\beta_{t}\underbrace{\left\Vert \Gamma_{2}\left(\Gamma_{1}^{\lambda}(\mu_{t}),\mu_{t}\right)-\Gamma_{2}\left(\Gamma_{1}^{\lambda}(\mu^{*}),\mu^{*}\right)\right\Vert _{\mH}}_{(a)}\nonumber \\
 & \qquad+\beta_{t}\underbrace{\left\Vert \Gamma_{2}\left(\Gamma_{1}^{\lambda}(\mu_{t}),\mu_{t}\right)-\Gamma_{2}(\pi_{t+1},\mu_{t})\right\Vert _{\mH}}_{(b)},\label{eq:mu_gap}
\end{align}
where the equality $(i)$ follows from the fact that $\mu^{*}=\Gamma_{2}\left(\Gamma_{1}^{\lambda}(\mu^{*}),\mu^{*}\right)$.

Lemma \ref{lem:Lipz_Lambda} implies that $\Lambda(\mu)=\Gamma_{2}\left(\Gamma_{1}^{\lambda}(\mu),\mu\right)$
is $d_{1}d_{2}+d_{3}$ Lipschitz. It follows that 
\begin{align}
(a)\le & \left(d_{1}d_{2}+d_{3}\right)\left\Vert \mu_{t}-\mu^{*}\right\Vert _{\H}.\label{eq:mu_gap_a}
\end{align}
By Assumption \ref{assu:Lipz_Gamma_2}, we have
\begin{align}
(b) & \le d_{2}D\left(\Gamma_{1}^{\lambda}(\mu_{t}),\pi_{t+1}\right).\label{eq:mu_gap_b}
\end{align}
Combining Eqs.$\ $(\ref{eq:mu_gap})-(\ref{eq:mu_gap_b}) yields
\begin{align}
\left\Vert \mu_{t+1}-\mu^{*}\right\Vert _{\H} & \le\left(1-\beta_{t}\overline{d}\right)\left\Vert \mu_{t}-\mu^{*}\right\Vert _{\H}+d_{2}\beta_{t}D\left(\Gamma_{1}^{\lambda}(\mu_{t}),\pi_{t+1}\right),\label{eq:mu_recur_1-1}
\end{align}
where $\overline{d}=1-d_{1}d_{2}-d_{3}>0$.

Let us bound the second RHS term above. By the definition of policy
distance $D$ in equation~(\ref{eq:def_policy_distance}), we have
\begin{align}
D\left(\Gamma_{1}^{\lambda}(\mu_{t}),\pi_{t+1}\right) & =\E_{\rho^{*}}\left[\left\Vert \Gamma_{1}^{\lambda}(\mu_{t})-\pi_{t+1}\right\Vert _{1}\right]\nonumber \\
 & =\E_{s\sim\rho^{*}}\left[\left\Vert \pi_{t+1}^{*}(\cdot|s)-\pi_{t+1}(\cdot|s)\right\Vert _{1}\right]\nonumber \\
 & =\E_{s\sim\rho_{t}^{*}}\left[\frac{\rho^{*}(s)}{\rho_{t}^{*}(s)}\left\Vert \pi_{t+1}^{*}(\cdot|s)-\pi_{t+1}(\cdot|s)\right\Vert _{1}\right]\nonumber \\
 & \leq\left\{ \E_{s\sim\rho_{t}^{*}}\left[\left|\frac{\rho^{*}(s)}{\rho_{t}^{*}(s)}\right|^{2}\right]\cdot\E_{s\sim\rho_{t}^{*}}\left[\left\Vert \pi_{t+1}^{*}(\cdot|s)-\pi_{t+1}(\cdot|s)\right\Vert _{1}^{2}\right]\right\} ^{1/2}\nonumber \\
 & \leq\overline{C}_{\rho}\sqrt{\E_{s\sim\rho_{t}^{*}}\left[\KL\left(\pi_{t+1}^{*}(\cdot|s)\Vert\pi_{t+1}(\cdot|s)\right)\right]},\label{eq:L1_KL_bound}
\end{align}
where the first inequality holds due to Cauchy-Schwartz inequality,
the last inequality follows from Assumption \ref{assu:concentrability}
and Pinsker's inequality.

Combining (\ref{eq:mu_recur_1-1})-(\ref{eq:L1_KL_bound}) gives 
\[
\left\Vert \mu_{t+1}-\mu^{*}\right\Vert _{\H}\le\left(1-\beta_{t}\overline{d}\right)\left\Vert \mu_{t}-\mu^{*}\right\Vert _{\H}+d_{2}\beta_{t}\overline{C}_{\rho}\sqrt{\E_{s\sim\rho_{t}^{*}}\left[\KL\left(\pi_{t+1}^{*}(\cdot|s)\Vert\pi_{t+1}(\cdot|s)\right)\right]}.
\]
This completes the proof.
\end{proof}

\section{Proof of Corollary~\ref{cor:main} \label{sec:proof_cor_main}}
\begin{proof}
Note that for each $t\in[T]$, we have 
\begin{align*}
D(\pi_{t},\pi^{*}) & \le D\left(\pi_{t},\pi_{t}^{*}\right)+D\left(\pi_{t}^{*},\pi^{*}\right)\\
 & =D\left(\pi_{t},\pi_{t}^{*}\right)+D\left(\Gamma_{1}^{\lambda}(\mu_{t}),\Gamma_{1}^{\lambda}(\mu^{*})\right)\\
 & \le D\left(\pi_{t},\pi_{t}^{*}\right)+d_{1}\left\Vert \mu_{t}-\mu^{*}\right\Vert _{\H},
\end{align*}
where the last step follows from Assumption~\ref{assu:Lipz_Gamma_1}
on the Lipschitzness of $\Gamma_{1}^{\lambda}$. It follows that 
\begin{align*}
 & D\left(\frac{1}{T}\sum_{t=1}^{T}\pi_{t},\pi^{*}\right)+\left\Vert \frac{1}{T}\sum_{t=1}^{T}\mu_{t}-\mu^{*}\right\Vert _{\mH}\\
\le & \frac{1}{T}\sum_{t=1}^{T}D\left(\pi_{t},\pi^{*}\right)+\frac{1}{T}\sum_{t=1}^{T}\left\Vert \mu_{t}-\mu^{*}\right\Vert _{\mH}\\
\le & \frac{1}{T}\sum_{t=1}^{T}\left(D\left(\pi_{t},\pi_{t}^{*}\right)+d_{1}\left\Vert \mu_{t}-\mu^{*}\right\Vert _{\H}\right)+\frac{1}{T}\sum_{t=1}^{T}\left\Vert \mu_{t}-\mu^{*}\right\Vert _{\mH}\\
\lesssim & \frac{1}{\sqrt{\lambda}}\left(\frac{\sqrt{\log T}}{T^{1/5}}+\sqrt{\varepsilon}\right),
\end{align*}
 where in the last step we apply the bounds~(\ref{eq:main_policy_bound})
and~(\ref{eq:main_mean_field_bound}) in Theorem~\ref{thm:main}.
\end{proof}

\section{Additional Proofs}

\subsection{Proof of Lemma \ref{lem:Lipz_Lambda} \label{subsec:proof_Lipz_Lambda}}
\begin{proof}
By the definition of $\Lambda$, we have 
\begin{align*}
 & \left\Vert \Lambda^{\lambda}(\mu)-\Lambda^{\lambda}(\mu')\right\Vert _{\H}\\
= & \left\Vert \Gamma_{2}\left(\Gamma_{1}^{\lambda}(\mu),\mu\right)-\Gamma_{2}\left(\Gamma_{1}^{\lambda}(\mu'),\mu'\right)\right\Vert _{\H}\\
\leq & \left\Vert \Gamma_{2}\left(\Gamma_{1}^{\lambda}(\mu),\mu\right)-\Gamma_{2}\left(\Gamma_{1}^{\lambda}(\mu'),\mu\right)\right\Vert _{\H}+\left\Vert \Gamma_{2}\left(\Gamma_{1}^{\lambda}(\mu'),\mu\right)-\Gamma_{2}\left(\Gamma_{1}^{\lambda}(\mu'),\mu'\right)\right\Vert _{\H} &  & \text{triangle inequality}\\
\leq & d_{2}D\left(\Gamma_{1}^{\lambda}(\mu),\Gamma_{1}^{\lambda}(\mu')\right)+d_{3}\left\Vert \mu-\mu'\right\Vert _{\H} &  & \text{Assumption }\ref{assu:Lipz_Gamma_2}\\
\leq & d_{1}d_{2}\left\Vert \mu-\mu'\right\Vert _{\text{\ensuremath{\H}}}+d_{3}\left\Vert \mu-\mu'\right\Vert _{\H}, &  & \text{Assumption \ref{assu:Lipz_Gamma_1}}
\end{align*}
which proves the lemma.
\end{proof}

\subsection{Proof of Lemma \ref{lem:performance_difference} \label{subsec:proof_performance_difference}}
\begin{proof}
By the definition of $V_{\mu}^{\lambda,\pi}$ in (\ref{eq:ER_MDP_L}),
we have 
\begin{align}
 & V_{\mu}^{\lambda,\pi'}(s)\nonumber \\
= & \E_{a_{t}\sim\pi'(s_{t}),s_{t+1}\sim\P(\cdot|s_{t},a_{t},\mu)}\left[\sum_{t=0}^{\infty}\gamma^{t}\left[r_{\mu}^{\lambda,\pi'}(s,a)+V_{\mu}^{\lambda,\pi}(s_{t})-V_{\mu}^{\lambda,\pi}(s_{t})\right]\mid s_{0}=s\right].\nonumber \\
= & \E_{a_{t}\sim\pi'(s_{t}),s_{t+1}\sim\P(\cdot|s_{t},a_{t},\mu)}\left[\sum_{t=0}^{\infty}\gamma^{t}\left[r_{\mu}^{\lambda,\pi'}(s,a)+\gamma V_{\mu}^{\lambda,\pi}(s_{t+1})-V_{\mu}^{\lambda,\pi}(s_{t})\right]\mid s_{0}=s\right]+V_{\mu}^{\lambda,\pi}(s).\label{eq:V_difference}
\end{align}
Recall that the Q-function $Q_{\mu}^{\lambda,\pi}$ of a policy $\pi$
for the regularized $\mdp_{\mu}$ is related to $V_{\mu}^{\lambda,\pi}$
as
\begin{align*}
V_{\mu}^{\lambda,\pi}(s) & =\E_{a\sim\pi(s)}\left[Q_{\mu}^{\lambda,\pi}(s,a)-\lambda\log\pi(a|s)\right]=\left\langle Q_{\mu}^{\lambda,\pi}(s,\cdot),\pi(\cdot|s)\right\rangle +\lambda\entro\left(\pi(\cdot|s)\right),\qquad\forall s\in\mS,\\
Q_{\mu}^{\lambda,\pi}(s,a) & =r(s,a,\mu)+\gamma\E_{s_{1}\sim\P(\cdot|s,a,\mu)}\left[V_{\mu}^{\lambda,\pi}(s_{1})\right],\qquad\forall(s,a)\in\mS\times\mA.
\end{align*}
We have

\begin{align*}
\left\langle Q_{\mu}^{\lambda,\pi}(s,\cdot),\pi'(\cdot|s)\right\rangle  & =\E_{a\sim\pi'(s)}\left[Q_{\mu}^{\lambda,\pi}(s,a)\right],\\
 & =\E_{a\sim\pi'(s)}\left[r(s,a,\mu)+\gamma\E_{s_{1}\sim\P(\cdot|s,a,\mu)}\left[V_{\mu}^{\lambda,\pi}(s_{1})\right]\right]\\
 & =\E_{a\sim\pi'(s),s_{1}\sim\P(\cdot|s,a,\mu)}\left[r_{\mu}^{\lambda,\pi'}(s,a)+\gamma V_{\mu}^{\lambda,\pi}(s_{1})+\lambda\log\pi'(a|s)\right]\\
 & =\E_{a\sim\pi'(s),s_{1}\sim\P(\cdot|s,a,\mu)}\left[r_{\mu}^{\lambda,\pi'}(s,a)+\gamma V_{\mu}^{\lambda,\pi}(s_{1})\right]-\lambda\entro\left(\pi'(\cdot|s)\right).
\end{align*}
Therefore, 
\begin{align}
 & \left\langle Q_{\mu}^{\lambda,\pi}(s,\cdot),\pi'(\cdot|s)-\pi(\cdot|s)\right\rangle \nonumber \\
= & \E_{a\sim\pi'(s),s_{1}\sim\P(\cdot|s,a,\mu)}\left[r^{\lambda,\pi'}(s,a,\mu)+\gamma V_{\mu}^{\lambda,\pi}(s_{1})\right]-\lambda\entro\left(\pi'(\cdot|s)\right)-V_{\mu}^{\lambda,\pi}(s)+\lambda\entro\left(\pi(\cdot|s)\right)\nonumber \\
= & \E_{a\sim\pi'(s),s_{1}\sim\P(\cdot|s,a,\mu)}\left[r^{\lambda,\pi'}(s,a,\mu)+\gamma V_{\mu}^{\lambda,\pi}(s_{1})-V_{\mu}^{\lambda,\pi}(s)\right]-\lambda\left[\entro\left(\pi'(\cdot|s)\right)-\entro\left(\pi(\cdot|s)\right)\right].\label{eq:Q_pi_difference}
\end{align}
Plugging (\ref{eq:Q_pi_difference}) into (\ref{eq:V_difference}),
we have 
\begin{align}
 & V_{\mu}^{\lambda,\pi'}(s)-V_{\mu}^{\lambda,\pi'}(s)\nonumber \\
= & \E_{a_{t}\sim\pi'(s_{t}),s_{t+1}\sim\P(\cdot|s_{t},a_{t},\mu)}\left[\sum_{t=0}^{\infty}\gamma^{t}\left\langle Q_{\mu}^{\lambda,\pi}(s_{t},\cdot),\pi'(\cdot|s_{t})-\pi(\cdot|s_{t})\right\rangle \mid s_{0}=s\right]\nonumber \\
 & +\E_{a_{t}\sim\pi'(s_{t}),s_{t+1}\sim\P(\cdot|s_{t},a_{t},\mu)}\left[\sum_{t=0}^{\infty}\gamma^{t}\lambda\left(\entro\left(\pi'(\cdot|s_{t})\right)-\entro\left(\pi(\cdot|s_{t})\right)\right)\mid s_{0}=s\right].\label{eq:V_difference_1}
\end{align}
Recall the definition of $J_{\mu}^{\lambda}(\pi)$ in (\ref{eq:expected_value_function}).
Taking expectation with respect to $s\sim\nu_{0}$  on both sides
of (\ref{eq:V_difference_1}) yields
\end{proof}
\begin{align}
 & J_{\mu}^{\lambda}(\pi')-J_{\mu}^{\lambda}(\pi)\nonumber \\
= & \E_{s_{0}\sim\nu_{0},a_{t}\sim\pi'(s_{t}),s_{t+1}\sim\P(\cdot|s_{t},a_{t},\mu)}\left[\sum_{t=0}^{\infty}\gamma^{t}\left\langle Q_{\mu}^{\lambda,\pi}(s_{t},\cdot),\pi'(\cdot|s_{t})-\pi(\cdot|s_{t})\right\rangle \right]\nonumber \\
\qquad & +\E_{s_{0}\sim\nu_{0},a_{t}\sim\pi'(s_{t}),s_{t+1}\sim\P(\cdot|s_{t},a_{t},\mu)}\left[\sum_{t=0}^{\infty}\gamma^{t}\lambda\left(\entro\left(\pi'(\cdot|s_{t})\right)-\entro\left(\pi(\cdot|s_{t})\right)\right)\right]\nonumber \\
= & \frac{1}{1-\gamma}\E_{s\sim\rho_{\mu}^{\pi'}}\left[\left\langle Q_{\mu}^{\lambda,\pi}(s,\cdot),\pi'(\cdot|s)-\pi(\cdot|s)\right\rangle +\lambda\left(\entro\left(\pi'(\cdot|s)\right)-\entro\left(\pi(\cdot|s)\right)\right)\right].\label{eq:J_difference}
\end{align}
For the entropy term in (\ref{eq:J_difference}), we have 
\begin{align}
 & \E_{s\sim\rho_{\mu}^{\pi'}}\left[\entro\left(\pi'(\cdot|s)\right)-\entro\left(\pi(\cdot|s)\right)\right]\nonumber \\
= & \E_{s\sim\rho_{\mu}^{\pi'}}\left[\left\langle \log\frac{1}{\pi'(\cdot|s)},\pi'(\cdot|s)\right\rangle -\left\langle \log\frac{1}{\pi(\cdot|s)},\pi(\cdot|s)\right\rangle \right]\nonumber \\
= & \E_{s\sim\rho_{\mu}^{\pi'}}\left[\left\langle \log\frac{1}{\pi(\cdot|s)}-\log\frac{\pi'(\cdot|s)}{\pi(\cdot|s)},\pi'(\cdot|s)\right\rangle -\left\langle \log\frac{1}{\pi(\cdot|s)},\pi(\cdot|s)\right\rangle \right]\nonumber \\
= & \E_{s\sim\rho_{\mu}^{\pi'}}\left[\left\langle \log\frac{1}{\pi(\cdot|s)},\pi'(\cdot|s)-\pi(\cdot|s)\right\rangle -\KL\left(\pi'(\cdot|s)\Vert\pi(\cdot|s)\right)\right].\label{eq:entropy_difference}
\end{align}
Taking (\ref{eq:entropy_difference}) into (\ref{eq:J_difference})
yields the desired equation in Lemma \ref{lem:performance_difference}.

\subsection{Proof of Lemma \ref{lem:optimal_ER_MDP} \label{subsec:proof_optimal_ER}}
\begin{proof}
Note that the value function $V_{\mu}^{\lambda,\pi}$ can be written
as 
\[
V_{\mu}^{\lambda,\pi}(s)=\E\left[\sum_{t=0}^{\infty}\gamma^{t}r_{\mu}^{\lambda,\pi}(s_{t},a_{t})|s_{0}=s\right].
\]
By the definition of $r_{\mu}^{\lambda,\pi}$ in (\ref{eq:ER_reward}),
we have $0\leq\E_{\pi}\left[r_{\mu}^{\lambda,\pi}(s_{t},a_{t})\right]\leq R_{\max}+\lambda\log|\mA|$.
Therefore, 
\[
0\leq V_{\mu}^{\lambda,\pi}(s)\leq\frac{R_{\max}+\lambda\log|\mA|}{1-\gamma},\qquad\forall s\in\mS,
\]
and 
\[
0\leq Q_{\mu}^{\lambda,\pi}(s,a)\leq R_{\max}+\gamma\frac{R_{\max}+\lambda\log|\mA|}{1-\gamma}=\frac{R_{\max}+\gamma\lambda\log|\mA|}{1-\gamma},\qquad\forall s\in\mS,a\in\mA.
\]

For the second inequality, we have 
\begin{align*}
\pi_{\mu}^{\lambda,*}(a|s) & =\frac{\exp\left(Q_{\mu}^{\lambda,*}(s,a)/\lambda\right)}{\sum_{b\in\mA}\exp\left(Q_{\mu}^{\lambda,*}(s,b)/\lambda\right)}\\
 & \ge\frac{1}{\sum_{b\in\mA}\exp\left(Q_{\max}/\lambda\right)}=\frac{1}{e^{Q_{\max}/\lambda}|\mA|}
\end{align*}
as claimed.
\end{proof}

\subsection{Proof of Lemma \ref{lem:one_step_MD} \label{subsec:proof_one_step_descent}}
\begin{proof}
For any function $g:\mA\rightarrow\R$ and distribution $p\in\simplex(\mA)$,
let $z:\mA\rightarrow\R$ be a constant function defined by 
\[
z(a)=\log\left(\sum_{a'\in\mA}p(a')\cdot\exp\left(\alpha g(a')\right)\right).
\]
Note that for any distributions $p^{*},p'\in\simplex(\mA),$$\left\langle z,p^{*}-p'\right\rangle =0$.
Since
\[
p'(\cdot)\propto p(\cdot)\cdot\exp\left(\alpha g(\cdot)\right),
\]
we have $\alpha g(\cdot)=z(\cdot)+\log(p'(\cdot)/p(\cdot)).$ Hence
\begin{align*}
\alpha\left\langle g,p^{*}-p'\right\rangle  & =\left\langle z+\log(p'/p),p^{*}-p'\right\rangle \\
 & =\left\langle z,p^{*}-p'\right\rangle +\left\langle \log(p^{*}/p),p^{*}\right\rangle +\left\langle \log(p'/p^{*}),p^{*}\right\rangle +\left\langle \log(p'/p),-p'\right\rangle \\
 & =\KL\left(p^{*}\Vert p\right)-\KL\left(p^{*}\Vert p'\right)-\KL\left(p'\Vert p\right).
\end{align*}
Therefore, for each state $s\in\mS$, we have
\begin{align*}
\alpha\left\langle G(s,\cdot),p^{*}-p\right\rangle  & =\alpha\left\langle G(s,\cdot),p^{*}-p'\right\rangle +\alpha\left\langle G(s,\cdot),p'-p\right\rangle \\
 & =\KL\left(p^{*}\Vert p\right)-\KL\left(p^{*}\Vert p'\right)-\KL\left(p'\Vert p\right)+\alpha\left\langle G(s,\cdot),p'-p\right\rangle \\
 & \leq\KL\left(p^{*}\Vert p\right)-\KL\left(p^{*}\Vert p'\right)-\KL\left(p'\Vert p\right)+\alpha\left\Vert G(s,\cdot)\right\Vert _{\infty}\cdot\left\Vert p-p'\right\Vert _{1}.
\end{align*}
Rearranging terms yields 
\begin{equation}
\KL\left(p^{*}\Vert p'\right)\leq\KL\left(p^{*}\Vert p\right)-\alpha\left\langle G(s,\cdot),p^{*}-p\right\rangle -\KL\left(p'\Vert p\right)+\alpha\left\Vert G(s,\cdot)\right\Vert _{\infty}\cdot\left\Vert p-p'\right\Vert _{1}.\label{eq:MD_1}
\end{equation}

Meanwhile, by Pinsker's inequality, it holds that 
\begin{equation}
\KL\left(p'\Vert p\right)\geq\left\Vert p-p'\right\Vert _{1}^{2}/2.\label{eq:pinsker}
\end{equation}
By combining (\ref{eq:MD_1}) and (\ref{eq:pinsker}), we obtain 
\begin{align*}
\KL\left(p^{*}\Vert p'\right) & \leq\KL\left(p^{*}\Vert p\right)-\alpha\left\langle G(s,\cdot),p^{*}-p\right\rangle -\left\Vert p-p'\right\Vert _{1}^{2}/2+\alpha\left\Vert G(s,\cdot)\right\Vert _{\infty}\cdot\left\Vert p-p'\right\Vert _{1}\\
 & \leq\KL\left(p^{*}\Vert p\right)-\alpha\left\langle G(s,\cdot),p^{*}-p\right\rangle +\alpha^{2}\left\Vert G(s,\cdot)\right\Vert _{\infty}^{2}/2,
\end{align*}
which concludes the proof.
\end{proof}

\section{Proofs of Theorem~\ref{thm:main_W} and Corollary~\ref{cor:main_W}}
\label{sec:proof_main_W}

The proof follows similar lines as those of Theorem~\ref{thm:main}
and Corollary~\ref{cor:main}, with all appearances of the distance
$D$ replaced by the new distance $W$. Below we only point out the
modifications needed. 

Lemma~\ref{lem:KL_recursion} remains valid as stated. For the proof
of this lemma, the only different step is bounding the term $B_{2}$
in equation~(\ref{eq:KL_t_bound}). In particular, the bounds in
equation~(\ref{eq:bound_B2}) should be replaced by the following:
\begin{align}
B_{2} & \leq\kappa\E_{s\sim\rho_{t}^{*}}\left[\left\Vert \pi_{t}^{*}(\cdot|s)-\pi_{t+1}^{*}(\cdot|s)\right\Vert _{1}\right]\nonumber \\
 & =\kappa\E_{s\sim\rho^{*}}\left[\frac{\rho_{t}^{*}(s)}{\rho^{*}(s)}\cdot\left\Vert \pi_{t}^{*}(\cdot|s)-\pi_{t+1}^{*}(\cdot|s)\right\Vert _{1}\right]\nonumber \\
 & \le\kappa\sqrt{\E_{s\sim\rho^{*}}\left[\left(\frac{\rho_{t}^{*}(s)}{\rho^{*}(s)}\right)^{2}\right]\cdot\E_{s\sim\rho^{*}}\left[\left\Vert \pi_{t}^{*}(\cdot|s)-\pi_{t+1}^{*}(\cdot|s)\right\Vert _{1}^{2}\right]}\nonumber \\
 & \leq\kappa C_{\rho}\cdot\sqrt{\E_{s\sim\rho^{*}}\left[\left\Vert \pi_{t}^{*}(\cdot|s)-\pi_{t+1}^{*}(\cdot|s)\right\Vert _{1}^{2}\right]} &  & \text{Assumption \ref{assu:concentrability_W}}\nonumber \\
 & =\kappa C_{\rho}W\left(\Gamma_{1}^{\lambda}(\mu_{t-1}),\Gamma_{1}^{\lambda}(\mu_{t})\right)\nonumber \\
 & \leq\kappa C_{\rho}d_{1}\left\Vert \mu_{t-1}-\mu_{t}\right\Vert _{\H}. &  & \text{Assumption \ref{assu:Lipz_Gamma_1_W}}\label{eq:bound_B2_W}
\end{align}

Lemma~\ref{lem:mu_recursion} should be replaced by the following
lemma.
\begin{lem}
\label{lem:mu_recursion_W}Under the setting of Theorem \ref{thm:main_W},
for each $t\geq0$, we have 
\[
\gapmu^{t+1}\text{\ensuremath{\leq}}\left(1-\beta_{t}\overline{d}\right)\gapmu^{t}+d_{2}\sqrt{\overline{C}_{\rho}}\beta_{t}\left(\gappi^{t+1}\right)^{1/4},
\]
where $\overline{d}=1-d_{1}d_{2}-d_{3}>0$.
\end{lem}
The proof of Lemma~\ref{lem:mu_recursion_W} is similar to that of
Lemma~\ref{lem:mu_recursion}. The only different step is the term
$D\left(\Gamma_{1}^{\lambda}(\mu_{t}),\pi_{t+1}\right)$ in equation~(\ref{eq:mu_recur_1-1})
should be replaced by $W\left(\Gamma_{1}^{\lambda}(\mu_{t}),\pi_{t+1}\right),$
which can be bounded as follows: 
\begin{align}
W\left(\Gamma_{1}^{\lambda}(\mu_{t}),\pi_{t+1}\right) 
 & =\sqrt{\E_{s\sim\rho^{*}}\left[\left\Vert \pi_{t+1}^{*}(\cdot|s)-\pi_{t+1}(\cdot|s)\right\Vert _{1}^{2}\right]}\nonumber \\
 & =\sqrt{\E_{s\sim\rho_{t}^{*}}\left[\frac{\rho^{*}(s)}{\rho_{t}^{*}(s)}\left\Vert \pi_{t+1}^{*}(\cdot|s)-\pi_{t+1}(\cdot|s)\right\Vert _{1}^{2}\right]}\nonumber \\
 & \leq\left\{ \E_{s\sim\rho_{t}^{*}}\left[\left|\frac{\rho^{*}(s)}{\rho_{t}^{*}(s)}\right|^{2}\right]\cdot\E_{s\sim\rho_{t}^{*}}\left[\left\Vert \pi_{t+1}^{*}(\cdot|s)-\pi_{t+1}(\cdot|s)\right\Vert _{1}^{4}\right]\right\} ^{1/4}\nonumber \\
 & \overset{(i)}{\lesssim}\sqrt{\overline{C}_{\rho}}\cdot\left\{ \E_{s\sim\rho_{t}^{*}}\left[\left\Vert \pi_{t+1}^{*}(\cdot|s)-\pi_{t+1}(\cdot|s)\right\Vert _{1}^{2}\right]\right\} ^{1/4}\nonumber \\
 & \overset{(ii)}{\lesssim}\sqrt{\overline{C}_{\rho}}\left\{ \E_{s\sim\rho_{t}^{*}}\left[\KL\left(\pi_{t+1}^{*}(\cdot|s)\Vert\pi_{t+1}(\cdot|s)\right)\right]\right\} ^{1/4}.\label{eq:L1_KL_bound_W}
\end{align}
where step $(i)$ holds by Assumption~\ref{assu:concentrability_W}
and the fact that $\left\Vert \nu-\nu'\right\Vert _{1}\le2,\forall\nu,\nu'\in\simplex(\mA)$,
and step $(ii)$ follows Pinsker's inequality.

We now turn to the proof of Theorem~\ref{thm:main_W}. 

We first establish the convergence for $\sigma_{\pi}^{t}$ by following
the exactly same steps from equation~(\ref{eq:KL_recur_1}) up to
equation~(\ref{eq:KL_bound-1}). We restate the bound on $\frac{1}{T}\sum_{t=0}^{T-1}\sigma_{\pi}^{t}$
in (\ref{eq:KL_bound-1}) as follows: 
\begin{equation}
\frac{1}{T}\sum_{t=0}^{T-1}\gappi^{t}\leq\frac{1}{T\lambda\alpha}\gappi^{0}+\frac{\overline{C}_{1}\beta}{\lambda\alpha}+\frac{2\varepsilon}{\lambda}+\frac{Q_{\max}^{2}}{2\lambda}\alpha+\frac{2\eta}{\lambda\alpha}.\label{eq:KL_bound_W}
\end{equation}
When choosing $\alpha=\bigO(T^{-4/9})$, $\beta=\bigO(T^{-8/9})$
and $\eta=\bigO(T^{-1})$, we have $\overline{C}_{1}=\bigO(\log T)$.
Therefore, we obtain
\begin{equation}
\frac{1}{T}\sum_{t=0}^{T-1}\gappi^{t}\lesssim\frac{\log T}{\lambda T^{4/9}}+\frac{2\varepsilon}{\lambda}.\label{eq:KL_convergence_W}
\end{equation}
If we let $\mathbb{\mathsf{T}}$ be a random number sampled uniformly
from $\{1,\ldots,T\},$ then the above equation can be written equivalently
as 
\begin{equation}
\E_{\mathsf{T}}\left[\sigma_{\pi}^{\mathsf{T}}\right]\lesssim\frac{\log T}{\lambda T^{4/9}}+\frac{2\varepsilon}{\lambda}.\label{eq:policy_bound_W}
\end{equation}

We now proceed to bound the average embedded mean-field state $\frac{1}{T}\sum_{t=0}^{T-1}\gapmu^{t}$.
Lemma~\ref{lem:mu_recursion_W} implies
\begin{equation}
\gapmu^{t}\le\frac{1}{\overline{d}\beta_{t}}\left(\gapmu^{t}-\gapmu^{t+1}\right)+\frac{d_{2}\sqrt{\overline{C}_{\rho}}}{\overline{d}}\left(\gappi^{t+1}\right)^{1/4}.\label{eq:mu_recur_W}
\end{equation}
With $\beta_{t}\equiv\beta=\bigO(T^{-8/9})$, averaging equation~(\ref{eq:mu_recur_W})
over iteration $t=0,\ldots,T-1$, we obtain 
\begin{align*}
\frac{1}{T}\sum_{t=0}^{T-1}\gapmu^{t} & \leq\frac{1}{\overline{d}\beta T}\left(\gapmu^{0}-\gapmu^{T}\right)+\frac{d_{2}\sqrt{\overline{C}_{\rho}}}{\overline{d}T}\sum_{t=0}^{T-1}\left(\gappi^{t+1}\right)^{1/4}\\
 & \leq\frac{\gapmu^{0}}{\overline{d}\beta T}+\frac{d_{2}\sqrt{\overline{C}_{\rho}}}{\overline{d}T}\sum_{t=0}^{T-1}\left(\gappi^{t+1}\right)^{1/4}\\
 & \overset{(i)}{\leq}\frac{\gapmu^{0}}{\overline{d}\beta T}+\frac{d_{2}\sqrt{\overline{C}_{\rho}}}{\overline{d}}\sqrt{\frac{1}{T}\sum_{t=0}^{T-1}\sqrt{\gappi^{t+1}}}\\
 & \overset{(ii)}{\leq}\frac{\gapmu^{0}}{\overline{d}\beta T}+\frac{d_{2}\sqrt{\overline{C}_{\rho}}}{\overline{d}}\left(\frac{1}{T}\sum_{t=0}^{T-1}\gappi^{t+1}\right)^{1/4}
\end{align*}
where steps $(i)$ and $(ii)$ follow from Cauchy-Schwarz inequality.

From equation (\ref{eq:KL_convergence_W}), we have
\begin{align*}
\frac{1}{T}\sum_{t=0}^{T-1}\gapmu^{t} & \lesssim\frac{\gapmu^{0}}{\overline{d}}T^{-1/9}+\frac{d_{2}\sqrt{\overline{C}_{\rho}}}{\overline{d}}\left(\frac{\log T}{\lambda T^{4/9}}+\frac{2\varepsilon}{\lambda}\right)^{1/4}\\
 & \lesssim\left(\frac{\log T}{\lambda T^{4/9}}+\frac{2\varepsilon}{\lambda}\right)^{1/4}\\
 & \lesssim\frac{1}{\lambda^{1/4}}\left(\frac{(\log T)^{1/4}}{T^{1/9}}+\varepsilon^{1/4}\right).
\end{align*}
This equation, together with Jensen's inequality, proves equation~(\ref{eq:main_mean_field_bound_W})
in Theorem~\ref{thm:main_W}. 

Turning to equation~(\ref{eq:main_policy_bound_W}) in Theorem~\ref{thm:main_W},
we have 
\begin{align*}
\frac{1}{T}\sum_{t=1}^{T}W\left(\pi_{t},\pi_{t}^{*}\right) & =\E_{\mathsf{T}}\left[W\left(\pi_{\mathsf{T}},\pi_{\mathsf{T}}^{*}\right)\right]\\
 & =\E_{\mathsf{T}}\sqrt{\E_{s\sim\rho^{*}}\left[\left\Vert \pi_{\mathsf{T}}^{*}(\cdot|s)-\pi_{\mathsf{T}}(\cdot|s)\right\Vert _{1}^{2}\right]}\\
 & \overset{(i)}{\leq}\sqrt{\E_{\mathsf{T}}\E_{s\sim\rho_{\mathsf{T}-1}^{*}}\left[\frac{\rho^{*}(s)}{\rho_{\mathsf{T}-1}^{*}(s)}\left\Vert \pi_{\mathsf{T}}^{*}(\cdot|s)-\pi_{\mathsf{T}}(\cdot|s)\right\Vert _{1}^{2}\right]}\\
 & \overset{(ii)}{\le}\left\{ \E_{\mathsf{T}}\E_{s\sim\rho_{\mathsf{T}-1}^{*}}\left[\left|\frac{\rho^{*}(s)}{\rho_{\mathsf{T}-1}^{*}(s)}\right|^{2}\right]\cdot\E_{\mathsf{T}}\E_{s\sim\rho_{\mathsf{T}-1}^{*}}\left[\left\Vert \pi_{\mathsf{T}}^{*}(\cdot|s)-\pi_{\mathsf{T}}(\cdot|s)\right\Vert _{1}^{4}\right]\right\} ^{1/4}\\
 & \overset{(iii)}{\lesssim}\left\{ \overline{C}_{\rho}^{2}\cdot\E_{\mathsf{T}}\E_{s\sim\rho_{\mathsf{T}-1}^{*}}\left[\left\Vert \pi_{\mathsf{T}}^{*}(\cdot|s)-\pi_{\mathsf{T}}(\cdot|s)\right\Vert _{1}^{2}\right]\right\} ^{1/4}\\
 & \overset{(iv)}{\lesssim}\sqrt{\overline{C}_{\rho}}\cdot\left\{ \E_{\mathsf{T}}\E_{s\sim\rho_{\mathsf{T}-1}^{*}}\left[\KL\left(\pi_{\mathsf{T}}^{*}(\cdot|s)\Vert\pi_{\mathsf{T}}(\cdot|s)\right)\right]\right\} ^{1/4}\\
 & =\sqrt{\overline{C}_{\rho}}\cdot\left\{ \E_{\mathsf{T}}\left[\sigma_{\pi}^{\mathsf{T}}\right]\right\} ^{1/4}\\
 & \overset{(v)}{\lesssim}\frac{1}{\lambda^{1/4}}\left(\frac{(\log T)^{1/4}}{T^{1/9}}+\varepsilon^{1/4}\right),
\end{align*}
where step $(i)$ holds due to Jensen's inequality, step $(ii)$ follows
from Cauchy-Schwarz inequality, step $(iii)$ follows from Assumption~\ref{assu:concentrability_W}
and the fact that $\left\Vert \nu-\nu'\right\Vert _{1}\le2,\forall\nu,\nu'\in\simplex(\mA)$,
step $(iv)$ comes from Pinsker's inequality, and step $(v)$ follows
from the bound in equation~(\ref{eq:policy_bound_W}). The above
equation, together with Jensen's inequality, proves equation~(\ref{eq:main_policy_bound_W}).
We have completed the proof of Theorem~\ref{thm:main_W}.

The proof of Corollary~\ref{cor:main_W} is the same as that of Corollary~\ref{cor:main}
and is omitted here.
\end{document}